\documentclass[runningheads]{llncs}

\usepackage{eccv}

\usepackage{eccvabbrv}

\usepackage{graphicx}
\usepackage{booktabs}
\usepackage{algorithm}
\usepackage{algorithmic}
\usepackage{microtype}
\usepackage{graphicx}
\usepackage{threeparttable}
\usepackage{subcaption}
\usepackage{booktabs} 
\usepackage{multirow}
\usepackage{comment}
\usepackage{tabularx}
\usepackage{xcolor}
\usepackage{makecell}
\usepackage{colortbl}
\usepackage{enumitem}
\usepackage{amsmath}
\usepackage{amssymb}
\usepackage{mathtools}
\usepackage{amsmath,amsfonts,bm}
\usepackage{esdiff}

\usepackage{color, colortbl}
\definecolor{LightCyan}{rgb}{0.88,1,1}
\usepackage{color,soul}
\usepackage[accsupp]{axessibility}  

\usepackage{hyperref}
\usepackage[capitalize,noabbrev]{cleveref}   
\usepackage{orcidlink}

\usepackage{amsmath,amsfonts,bm}
\usepackage{esdiff}

\newcommand{\TTA}[0]{Test Time Adaptation\xspace}
\newcommand{\tta}[0]{test time adaptation\xspace}
\newcommand{\ttaab}[0]{TTA\xspace}

\newcommand{\deyoabb}[0]{DeYO\xspace}
\newcommand{\ens}[0]{Naive\xspace}
\newcommand{\grad}[0]{Grad\xspace}
\newcommand{\svgd}[0]{SVGD\xspace}
\newcommand{\kl}[0]{KL\xspace}

\def\figref#1{\cref{#1}}

\def\eqref#1{\cref{#1}}

\def\1{\bm{1}}

\def\rvp{{\mathbf{p}}}

\def\rvx{{\mathbf{x}}}

\def\vtheta{{\bm{\theta}}}

\def\vf{{\bm{f}}}

\def\gD{{\mathcal{D}}}

\def\gX{{\mathcal{X}}}
\def\gY{{\mathcal{Y}}}

\newcommand{\softmax}{\mathrm{softmax}}

\begin{document}

\title{Multi-Hypothesis Test-Time Adaptation to Mitigate Underspecification} 

\titlerunning{Multi-Hypothesis TTA}


\author{Afshar Shamsi\inst{1} \and
Xiao-Yu Guo\inst{2} \and
Hamid Alinejad-Rokny\inst{3} \\
Arash Mohammadi\inst{1} \and
Damien Teney\inst{4} \and
Ehsan Abbasnejad\inst{5} }



\authorrunning{A.~Shamsi et al.}

\institute{Concordia University, Canada\\ 
\email{\{afshar.shamsi,arash.mohammadi\}@concordia.ca}\and
Australian Institute for Machine Learning, Australia\\
\email{xiaoyu.guo@gmail.com}\and
University of New South Wales, Australia\\
\email{h.alinejad@unsw.edu.au}\and
Idiap Research Institute, Switzerland\\
\email{damien.teney@idiap.ch}\and
Monash University, Australia\\
\email{Ehsan.Abbasnejad@monash.edu}}

\maketitle

\begin{abstract}
Test-Time Adaptation (TTA) seeks to improve model robustness under distribution shifts by adapting parameters using unlabeled target data. However, in the absence of supervision, entropy-based adaptation is fundamentally underconstrained: multiple distinct parameter updates can achieve similarly low entropy while inducing drastically different decision boundaries. This phenomenon, known as underspecification, renders standard TTA brittle and prone to collapse into spurious modes. In this work, we reinterpret TTA through a posterior-inspired lens induced by entropy minimization, where low-entropy solutions define a pseudo-likelihood over parameters. Instead of committing to a single point estimate, we introduce a particle-based diversification framework that explores multiple plausible adaptation trajectories simultaneously. Our method can be viewed as a structured exploration of multiple plausible adaptation solutions, implemented through multi-level diversification at the output, parameter, optimizer, and input levels. Crucially, the framework acts as a plug-and-play wrapper compatible with existing TTA methods. Extensive experiments on challenging benchmarks demonstrate consistent gains in stability and robustness, achieving improvements of 3–4\% under mixed shifts, 2–3\% with batch size one, and 1–2.5\% under label shifts, outperforming state-of-the-art baselines. Our results suggest that treating TTA as a multi-hypothesis inference problem, rather than a single-point optimization task, is key to mitigating underspecification and enabling reliable real-world deployment.
\end{abstract}

%
\section{Introduction}
\label{sec:intro}

\begin{figure}[!htb]
    \includegraphics[scale = 0.5]{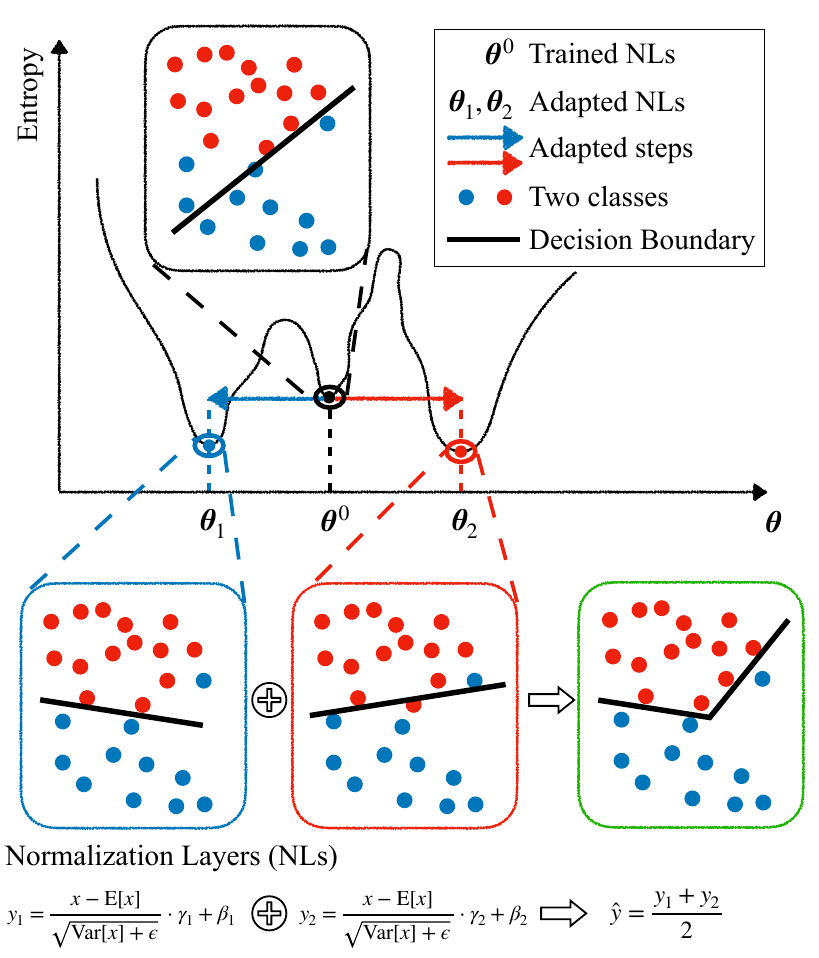}
    \centering
    \caption{Conceptual illustration of multi-hypothesis test-time adaptation. Entropy minimization can yield multiple low-entropy solutions in the adaptation landscape. Standard TTA follows a single trajectory from
    $\theta^0$, which may converge to a suboptimal decision boundary.
    Our framework instead maintains multiple adaptation particles
    ($\theta_1$, $\theta_2$) by adapting separate normalization
    parameters. Aggregating their predictions produces a more stable
    adaptation and mitigates underspecification.}
    \label{fig:motivation}
\end{figure}
Deep learning models have achieved remarkable performance across a range of tasks when training and test data are drawn from the same distribution. However, their performance often degrades sharply when deployed in the wild, where distribution shifts are inevitable \cite{quinonero2009dataset, recht2019imagenet}. Test-time adaptation (TTA) has emerged as a promising strategy to bridge this train--test gap by adapting the model to the target distribution using only unlabeled test data \cite{sun2020test, wang2021tent, zhang2022memo}. 

Despite its appeal, standard TTA optimizes a single entropy-based objective in the absence of ground-truth supervision. This renders the adaptation process fundamentally underconstrained: multiple parameter configurations can achieve similarly low entropy on the target data while corresponding to qualitatively different decision boundaries. This phenomenon, underspecification, has been shown to induce instability and spurious feature reliance in deep models. From an optimization perspective, entropy minimization may admit multiple local minima of comparable value, yet these solutions need not generalize equally well under shift. We argue that existing TTA methods implicitly compute a single maximum a posteriori (MAP) solution under an entropy-induced pseudo-likelihood. In underspecified regimes, where multiple low-entropy configurations exist, such point-estimate adaptation is inherently unstable: small perturbations in optimization trajectory can lead to qualitatively different decision boundaries. Consequently, entropy minimization under shift should be viewed not as a well-posed optimization problem, but as a multi-hypothesis inference problem.

\begin{figure*}[!htb]
    \includegraphics[width=0.9\linewidth]{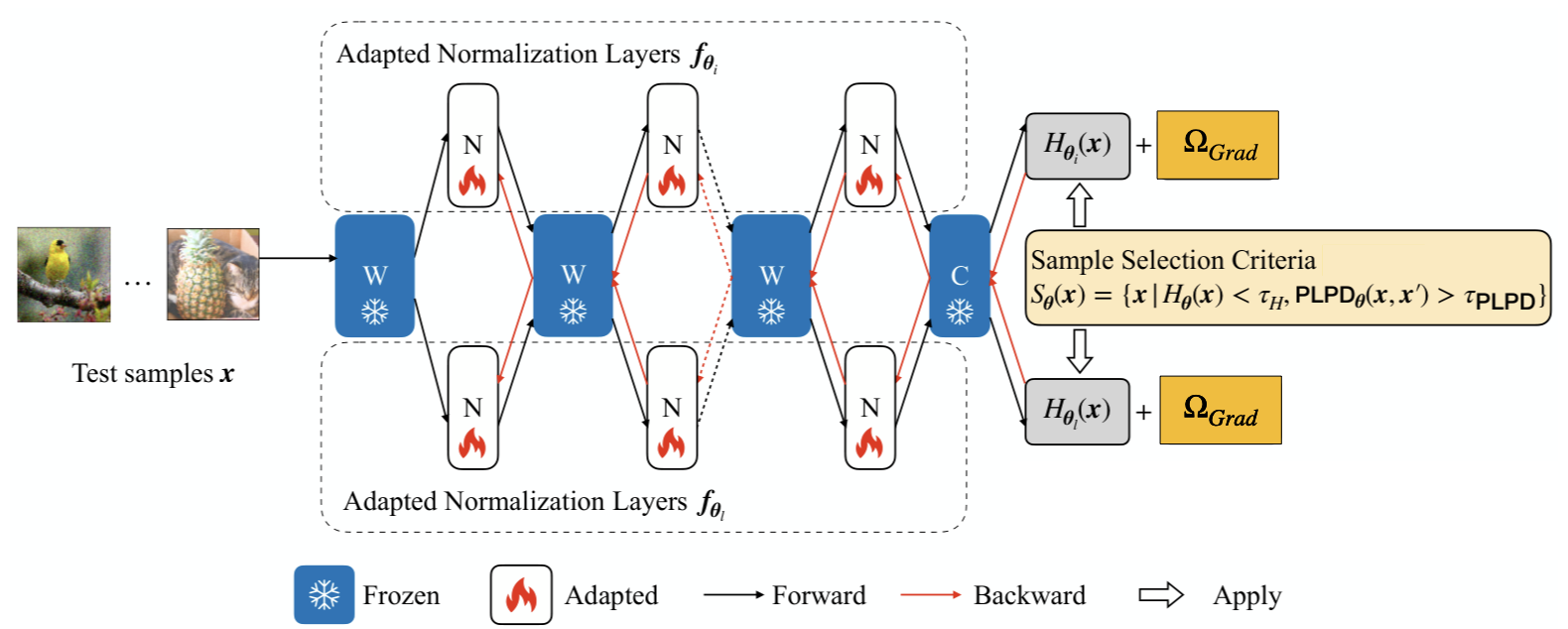}
    \centering
    \caption{An illustration of the proposed method for the case of gradient diversity diversification. During the adaptation, we only update the normalization layers (N) and keep the rest of layers (W and C) frozen. When a batch of test samples come, we first identify non-harmful ones (see Appendix~\ref{append:hyperparameter}). We only perform backward propagation on selected samples with gradient diversity measure to push normalization layers away from one another.}
    \label{fig:structure}
\end{figure*}

In this work, we reformulate TTA as a population-based inference problem and introduce a particle-driven diversification framework to approximate the posterior over adaptation parameters induced by entropy minimization. Rather than committing to a single entropy-minimizing trajectory, we maintain a population of interacting adaptation particles that explore distinct low-entropy regions of the parameter space while being regularized to prevent premature mode collapse. Each particle follows a unique optimization trajectory under entropy minimization, guided by multi-level diversification mechanisms that operate at the output, parameter, optimizer, and input levels. This perspective transforms TTA from a brittle point-estimation procedure into a controlled multi-trajectory exploration process. By preserving multiple plausible adaptation hypotheses and aggregating their predictions, our framework mitigates instability arising from underspecification and reduces reliance on spurious shortcuts. Importantly, the proposed approach is model-agnostic and acts as a plug-and-play wrapper that can enhance a broad class of existing TTA methods. We evaluate our framework on challenging benchmarks such as ImageNet-C~\cite{hendrycks2019benchmarking} under mixed corruptions, label shifts, and batch size one settings. Across all scenarios, our method consistently improves robustness and stability, outperforming state-of-the-art TTA baselines by 1–4\%. These results suggest that treating TTA as a multi-hypothesis inference problem is crucial for reliable deployment under distribution shift.

In \figref{fig:motivation}, we illustrate the entropy landscape with respect to the model's parameters $\vtheta$ for a binary classification task.
Suppose that $\vtheta^0$ represents all parameters of the normalization layers in the source model. Without adaptation, the decision boundary is likely to be inadequate under Out-of-Distribtion (OOD) conditions. At test time, \ttaab approaches typically update the model by minimizing entropy to align target data with source data.
Prior \ttaab approaches update $\vtheta^0$ to either $\vtheta_1$ or $\vtheta_2$, optimizing the model toward low-entropy directions. Although these methods can improve overall performance, their decision boundaries remain suboptimal and can introduce additional errors. These issues are closely associated with the broader challenge of underspecification, which has been mitigated in other contexts using strategies such as data augmentation and ensembling. However, their integration into \ttaab remains largely unexplored. Although \cite{deyo2024} indirectly used augmentation-based filtering to highlight robust features, this was not aimed at resolving underspecification itself. In contrast, our work directly tackles the root of underspecification by introducing a diversification-based adaptation framework that encourages representational diversity and enhances model robustness during \ttaab.

In this paper, we build on these observations to explain why existing \ttaab methods can be unreliable and to develop a more stable alternative.
Specifically, starting from a source model with normalization parameters $\vtheta^0$, we 1) initialize a collection of normalization parameter sets $\Theta = \{\vtheta_i\}_{i=1}^N$ from $\vtheta^0$ while keeping all non-normalization layers frozen (see \figref{fig:structure}), 2) use $N$ optimizers with distinct hyperparameters to adapt different particles $\vtheta_i$, and 3) apply diversification regularization to prevent particle collapse. At inference time, we aggregate predictions from the $N$ adapted parameter sets $\{\vtheta_i\}_{i=1}^N$, thereby combining multiple plausible adaptation hypotheses instead of committing to a single adapted solution. Each $\vtheta_i$ follows a distinct entropy-minimization trajectory, and their interaction encourages exploration of different low-entropy regions. Aggregating these interacting particles reduces sensitivity to spurious modes induced by underspecification. Empirically, this population-based inference consistently improves stability and robustness over single-model entropy minimization, outperforming representative \ttaab baselines including Tent~\cite{wang2021tent}, SAR~\cite{niu2023towards}, and \deyoabb~\cite{deyo2024}.

\vspace{0.15cm}
\noindent\textbf{Summary of Contributions:}
\begin{itemize}
    \item We reinterpret entropy-based test-time adaptation as an implicit posterior inference problem under underspecification, highlighting the limitations of single-point MAP adaptation.
    \item We propose a particle-based diversification framework that approximates posterior exploration over adaptation parameters through multi-level diversification (output, parameter, optimizer, and input levels).
    \item We demonstrate that controlled particle interactions mitigate collapse into spurious modes and consistently improve robustness across mixed shifts, label shifts, and batch size one scenarios.
    \item Extensive ablation studies reveal that gradient-based diversification provides the strongest stabilization effect, offering practical guidance for reliable test-time adaptation.
\end{itemize}
\section{Preliminaries}
\label{sec:background}

\subsection{\TTA}
\label{sec:tta}

Suppose that we have a training (source) dataset $\gD^{s} = \{(\rvx_j^{s}, y_j^{s})\}_{j=1}^{N^{s}}$ where $\rvx_j^{s} \in \gX^{s}$ and $y_j^{s} \in \gY^{s}$, $N^{s}$ is the number of training instances, and a testing (target) set $\gD^{t} = \{(\rvx_j^{t}, y_j^{t})\}_{j=1}^{N^{t}}$ where $\rvx_j^{t} \in \gX^{t}$ and $y_j^{t} \in \gY^{t}$, $N^{t}$ is the number of test instances.
A source model $\vf_{\vtheta}(\cdot)$ with parameters $\vtheta$ is first trained on $\gD^{s}$ and then evaluated on $\gD^{t}$. When $\gD^{t}$ involves OOD, $\vf_{\vtheta}(\cdot)$ may not generalize well. \TTA approaches adapt the model $\vf_{\vtheta}(\cdot)$ on the target dataset $\gD^{t} = \{\rvx_j^{t}\}_{j=1}^{N^{t}}$ without any labels, since labels $y_j^{t}$ are not accessible at test time.
\textcolor{black}{In this paper, we focus on \ttaab approaches that only reuse the normalization layers of the source model; therefore, without loss of generality, we use $\vtheta$ to represent all parameters of the normalization layers.}

\begin{figure}[!tb]
    \includegraphics[scale = 0.2]{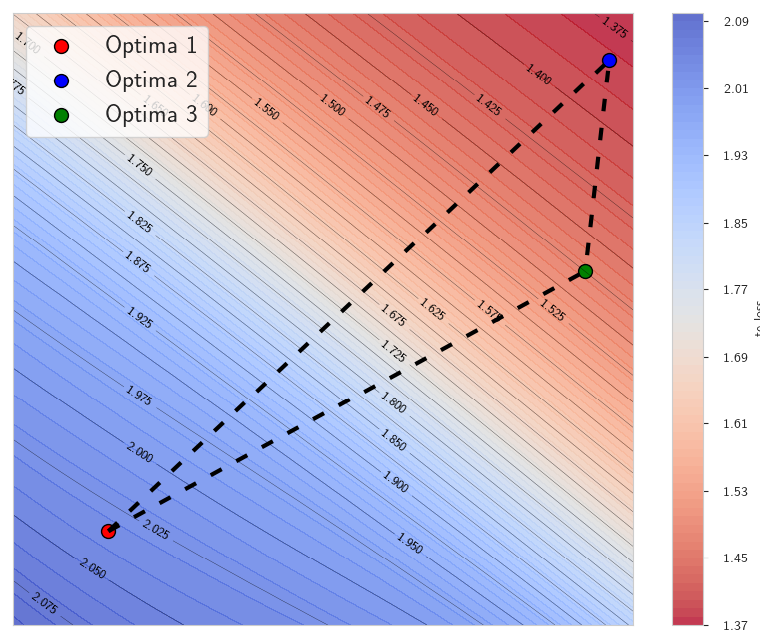}
    \centering
    \caption{The interpolation of three optima in the test loss landscape is visualized, where blue regions indicate high error/loss and red regions correspond to low error/loss.}
    \label{fig:interpolation}
\end{figure}

\subsection{Entropy Minimization}
Without the need for additional labeled data, previous \ttaab methods~\cite{wang2021tent,niu2023towards} utilize the entropy of model predictions as the objective function to adapt the model $\vf_{\vtheta}(\cdot)$:
\begin{equation}
    \min H_\vtheta (\gX^{t}) = \min \sum^{N^{t}}_{j} H(\vf_{\vtheta} (\rvx_j^{t})),
\label{eq:entropy minimization}
\end{equation}
where $H_\vtheta(\cdot)$ is entropy and measured by the Shannon entropy~\cite{shannon}:
\begin{equation}
    H_\vtheta (\rvx) = -\rvp_\vtheta(\rvx) \cdot \log \rvp_\vtheta (\rvx) = - \sum^{C}_{c} p_\vtheta(x)_c \log p_\vtheta(x)_c,
\label{eq:entropy}
\end{equation}
where $\rvp_\vtheta(\rvx) = \softmax(\vf_{\vtheta}(\rvx)) = (p_\vtheta(x)_1, ..., p_\vtheta(x)_C)$, $C$ is the number of class labels for the underlying classification task. 
Entropy is an unsupervised objective since it only depends on the model predictions instead of annotations.

\subsection{Benchmark Methods and Datasets}

For comparison, we conducted experiments on several benchmark datasets: 
\textbf{1) The ImageNet-C}~\cite{imagenetc2019} benchmark consists of 15 diverse corruption types encompassing 5 severity levels for each type applied to validation images of ImageNet. 
\textbf{2) CIFAR-10-C and CIFAR-100-C}~\cite{imagenetc2019} benchmarks provide the same 15 corruption types as ImageNet-C, but all images are derived from CIFAR-10 and CIFAR-100, respectively.
We applied our method to mild scenarios proposed by~\cite{wang2021tent} and wild scenarios proposed by~\cite{niu2023towards}.
The mild scenario is to adapt models with a batch of test samples that have the same distribution shift type and randomly shuffled label distribution.
In contrast, the wild scenario includes three more realistic settings: a) dynamic shifts in the ground-truth test label distribution, which lead to class imbalance at each corruption, b) batch size equal to one (single-sample adaptation), and c) mixtures of multiple distribution shifts.
We evaluate three representative \tta baseline methods including Tent~\cite{wang2021tent}, SAR~\cite{niu2023towards} and \deyoabb~\cite{deyo2024}, and employ ResNet-50-GN (group normalization) and ViTBase-LN (layer normalization)~\cite{vit2021} for ImageNet-C, CLIP-VitBase-Patch32-LN~\cite{clip2021} for CIFAR-10-C and CIFAR-100-C, taking into consideration \tta's utilization across various normalization layers. To further assess robustness under broader distribution shifts, we additionally evaluate on ColoredMNIST, Waterbirds, ImageNet-R~\cite{hendrycks2021many}, and VisDA-2021~\cite{bashkirova2022visda}, which capture spurious correlations and domain shifts beyond corruption-based benchmarks. Detailed results are provided in Appendix~\ref{append:beyond_corruptions}.

\section{Methodology and Results}
\label{sec:method}

\begin{table}[t]
\caption{Multi-seed DeYO solutions under batch size one (Zoom Blur, severity 5). Despite nearly identical final entropy values, different random seeds lead to substantial variance in final accuracy and distinct parameter endpoints. This empirical evidence supports the presence of underspecification in entropy-based test-time adaptation.}
\label{tab:multiseed_deyo_b1}
\centering
\small
\setlength{\tabcolsep}{6pt}
\scalebox{0.84}{
\begin{tabular}{lcccc}
\toprule
Seed & Final Entropy $\downarrow$ & Acc. (\%) $\uparrow$ & $\|\theta^{\mathrm{final}}-\theta^{0}\|_2$ & Mean Pairwise $\|\Delta \theta\|_2$ \\
\midrule
0  & 0.247 & 45.1 & 0.92 & 0.51 \\
1  & 0.241 & 34.6 & 0.88 & 0.47 \\
2  & 0.245 & 41.8 & 0.95 & 0.54 \\
3  & 0.249 & 32.9 & 0.90 & 0.49 \\
4  & 0.244 & 37.5 & 0.93 & 0.52 \\
\midrule
Mean $\pm$ Std & 0.245 $\pm$ 0.003 & 38.37 $\pm$ 4.80 & 0.92 $\pm$ 0.03 & 0.51 $\pm$ 0.03 \\
\bottomrule
\end{tabular}
}
\end{table}
Prior \ttaab approaches adapt a single model at test time by minimizing an unsupervised objective, most commonly, prediction entropy~\cite{wang2021tent}, or by combining entropy minimization with augmentation-based filtering~\cite{lee2024entropy, chen2022contrastive, deyo2024}. 
While these strategies have demonstrated promising improvements under distribution shift, they rely on updating a single parameter instance in the absence of supervision. 
As discussed in~\cref{sec:intro}, entropy minimization is inherently underconstrained: multiple parameter configurations can achieve similar low entropy while corresponding to qualitatively different decision boundaries. 
Updating a single model therefore risks drifting toward suboptimal low-entropy solutions, amplifying simplicity bias and accumulating adaptation errors over time.

To empirically evaluate the stability of single-model entropy-based adaptation, we run \deyoabb{}~\cite{deyo2024} under the batch size $1$ protocol multiple times with different random seeds while keeping the hyperparameters fixed. \Cref{tab:multiseed_deyo_b1} shows that the final entropy values are nearly identical across runs ($0.245 \pm 0.003$), yet the resulting accuracies vary substantially ($38.37 \pm 4.80$\%). Furthermore, the adapted normalization parameters converge to different endpoints in parameter space. In particular, $\|\boldsymbol{\theta}^{\mathrm{final}}-\boldsymbol{\theta}^{0}\|_2$ measures the magnitude of the parameter change from the source model, while the mean pairwise distance $\|\Delta\boldsymbol{\theta}\|_2$ denotes the average $\ell_2$ distance between final parameter vectors obtained from different seeds (computed across all seed pairs). 
These results indicate that entropy minimization does not uniquely specify a robust adaptation outcome: multiple low-entropy solutions correspond to qualitatively different decision boundaries and generalization performance. This instability highlights the need for structured mechanisms that explicitly maintain diverse adaptation hypotheses. To address this limitation, we adopt a population-based formulation of test-time adaptation. 
Instead of committing to a single adaptation trajectory, we maintain a collection of $K$ interacting adaptation particles 
\[
\Theta = \{\theta_i\}_{i=1}^{K},
\]
where each $\theta_i$ represents the normalization parameters adapted from the shared source model $\theta^0$. 
All non-normalization layers remain frozen and shared across particles, following standard \ttaab practice~\cite{wang2021tent, niu2023towards, deyo2024}. 
Each particle is initialized from $\theta^0$ and adapted online on the target stream. We define the general population adaptation objective as
\begin{equation}
\mathcal{L}(X)
=
\frac{1}{K}
\sum_{i=1}^{K}
\ell(X; \theta_i)
+
\lambda\,\Omega(\Theta),
\label{eq:pop_objective}
\end{equation}
where $\ell(X; \theta_i)$ denotes the entropy-based adaptation loss for particle $i$ (Eq.~\ref{eq:entropy minimization}), and $\Omega(\Theta)$ is a diversification regularizer that prevents premature particle collapse. The key idea is to encourage controlled exploration of multiple low-entropy regions in parameter space while maintaining interaction between particles. 
Unlike naive ensembling, where models are trained independently and averaged post hoc, our formulation explicitly couples particle dynamics during adaptation. 
Diversification therefore acts not as a static aggregation trick, but as a mechanism for mitigating underspecification during optimization.

We instantiate $\Omega(\Theta)$ through complementary diversification mechanisms operating at multiple levels: output-level, parameter-level, optimizer-level, and input-level. 
Each mechanism targets a distinct failure mode of single-model entropy minimization under distribution shift, as detailed below.

\subsection{Output-level Diversification}

Output-level diversification promotes disagreement among particles at the prediction level, preventing premature collapse to identical predictive hypotheses under entropy minimization. 
Let $p_i(X)$ denote the predictive distribution of particle $\theta_i$ on input batch $X$. 
To encourage predictive diversity, we maximize pairwise divergence between particle outputs. We define the output-level regularizer as:
\begin{equation}
\Omega_{\mathrm{KL}}(\Theta)
=
- \frac{1}{K(K-1)}
\sum_{i \neq j}
D_{\mathrm{KL}}\!\left(
p_i(X) \,\|\, p_j(X)
\right),
\label{eq:kl_div}
\end{equation}
where $D_{\mathrm{KL}}(\cdot \| \cdot)$ denotes the Kullback–Leibler divergence. The negative sign ensures that minimizing the overall objective in~\eqref{eq:pop_objective} increases predictive divergence between particles. This repulsive interaction encourages particles to occupy distinct low-entropy regions of the solution space, thereby reducing the risk of collective collapse into the same spurious decision boundary. Unlike naive ensembling, where models are trained independently and averaged post hoc, our formulation explicitly couples particle dynamics during adaptation, allowing disagreement to emerge as a controlled exploration mechanism. While naive ensembling already improves robustness over single-model adaptation, we empirically observe that independently trained particles tend to converge toward similar adaptation trajectories under entropy minimization. This trajectory alignment limits the diversity benefits of simple averaging. The proposed diversification mechanisms explicitly prevent such collapse and yield consistent gains beyond naive ensembling.

\subsection{Parameter-level Diversification}

\noindent \textbf{3.2.1 Stein Variational Gradient Descent~\cite{liu2016svgd}:} To further promote structured exploration of multiple adaptation hypotheses, we employ Stein Variational Gradient Descent (SVGD) as a particle interaction mechanism. Unlike standard gradient descent that converges to a single solution, SVGD evolves a set of particles jointly to approximate a target distribution through deterministic updates with repulsive interactions. We define a pseudo-posterior over normalization parameters as
\begin{equation}
\tilde{p}(\theta)
\propto
\exp\!\left(- \ell(X; \theta)\right)
\exp\!\left(-\frac{\gamma}{2} \|\theta - \theta^0\|_2^2 \right),
\end{equation}
where $\ell(X; \theta)$ is the entropy adaptation loss (Eq.~\ref{eq:entropy minimization}). The second exponential term acts as an implicit Gaussian prior centered at the source normalization parameters $\theta^0$, discouraging excessive drift during adaptation. The coefficient $\gamma > 0$ controls the strength of this quadratic anchoring term, balancing adaptation flexibility and retention of source-domain inductive bias. SVGD updates each particle $\theta_i$ according to
\begin{equation}
\theta_i \leftarrow
\theta_i +
\epsilon \, \phi(\theta_i),
\end{equation}
where $\epsilon > 0$ is the particle update step size, and
\begin{equation}
\phi(\theta_i)
=
\frac{1}{K}
\sum_{j=1}^{K}
\left[
\kappa(\theta_j, \theta_i)\,
\nabla_{\theta_j} \log \tilde{p}(\theta_j)
+
\nabla_{\theta_j} \kappa(\theta_j, \theta_i)
\right].
\end{equation}

\noindent Here, $K$ denotes the number of particles and $\kappa(\cdot,\cdot)$ is a positive-definite kernel (e.g., an RBF kernel) that induces repulsive interactions between particles. The first term drives particles toward low-entropy regions of the adaptation objective, while the second term prevents collapse by encouraging dispersion in parameter space. In the context of test-time adaptation, this joint update enables particles to explore multiple plausible low-entropy solutions, thereby mitigating underspecification and reducing sensitivity to spurious adaptation modes.

\noindent \textbf{3.2.2 Gradient-based Diversification:} To mitigate simplicity bias during test-time adaptation, we promote \emph{functional diversity} by encouraging particles to respond differently to input perturbations. Instead of enforcing parameter-space separation alone, we directly regularize the alignment of input gradients across particles. For particle $\theta_i$, let
\[
g_i = \nabla_X \ell(X; \theta_i)
\]
denote the gradient of the entropy loss with respect to the input batch $X$. 
We define the gradient-based diversification term as
\begin{equation}
\Omega_{\mathrm{Grad}}(\Theta)
=
- \frac{1}{K(K-1)}
\sum_{i \neq j}
\langle g_i, g_j \rangle,
\label{eq:grad_div}
\end{equation}
where $\langle \cdot , \cdot \rangle$ denotes the Euclidean inner product. Minimizing the overall objective in Eq.~\eqref{eq:pop_objective} therefore penalizes large positive gradient alignment between particles. This encourages particles to rely on different input directions during adaptation, promoting diverse decision boundaries under distribution shift. Unlike output-level disagreement, which acts on predictions, or parameter-space repulsion, which operates directly in weight space, gradient-based diversification regularizes the \emph{functional behavior} of particles. By reducing gradient alignment, particles are discouraged from collapsing toward identical adaptation trajectories, thereby improving robustness in underspecified settings. We thoroughly investigate output-level and parameter-level diversification under wild settings; results for mild settings are provided in Appendix~\ref{append:mild}.

\begin{table*}[!ht]
    \caption{Comparisons with baseline methods on ImageNet-C wild senarios at severity level 5 under batch size 1 and label shifts averaged over five random seeds regarding accuracy (\%).}
    \label{tab:imagenet-c}
    \centering
    \setlength{\tabcolsep}{1.6pt}
    \scalebox{0.485}{
        \begin{tabular}{l|ccc|cccc|cccc|cccc|c}
            \multicolumn{1}{c}{} & \multicolumn{3}{c}{Noise} & \multicolumn{4}{c}{Blur} & \multicolumn{4}{c}{Weather} & \multicolumn{4}{c}{Digital} & \\
            \textbf{Batch Size 1} & \textbf{Gauss.} & \textbf{Shot} & \textbf{Impl.} & \textbf{Defoc.} & \textbf{Glass} & \textbf{Motion} & \textbf{Zoom} & \textbf{Snow} & \textbf{Frost} & \textbf{Fog} & \textbf{Brit.} & \textbf{Contr.} & \textbf{Elastic} & \textbf{Pixel} & \textbf{JPEG} & \textbf{Avg.} \\
            \hline \hline
            ResNet-50-GN &18.01 &19.78 &17.91 &19.80 &11.41 &21.39 &24.90 &40.41 &47.29 &33.60 &69.31 &36.31 &18.59 &28.41 &52.29&30.60 \\
            {Tent}   & 28.85\textsubscript{\scriptsize{$\pm$0.4}} & 35.95\textsubscript{\scriptsize{$\pm$0.3}} & 33.49\textsubscript{\scriptsize{$\pm$0.3}} & 16.67\textsubscript{\scriptsize{$\pm$0.2}} & 7.19\textsubscript{\scriptsize{$\pm$0.3}}  & 27.35\textsubscript{\scriptsize{$\pm$0.3}} & 30.14\textsubscript{\scriptsize{$\pm$0.4}} & 17.97\textsubscript{\scriptsize{$\pm$0.2}} & 25.20\textsubscript{\scriptsize{$\pm$0.2}} & 2.27\textsubscript{\scriptsize{$\pm$0.2}}  & 71.88\textsubscript{\scriptsize{$\pm$0.1}} & 46.24\textsubscript{\scriptsize{$\pm$0.1}} & 7.36\textsubscript{\scriptsize{$\pm$0.3}}  & 52.72\textsubscript{\scriptsize{$\pm$0.3}} & 56.40\textsubscript{\scriptsize{$\pm$0.1}} & 30.65 \\
            {SAR}    & 28.51\textsubscript{\scriptsize{$\pm$0.3}} & 30.91\textsubscript{\scriptsize{$\pm$0.4}} & 29.27\textsubscript{\scriptsize{$\pm$0.1}} & 18.50\textsubscript{\scriptsize{$\pm$0.1}} & 15.34\textsubscript{\scriptsize{$\pm$0.3}} & 28.80\textsubscript{\scriptsize{$\pm$0.3}} & 30.53\textsubscript{\scriptsize{$\pm$0.2}} & 44.47\textsubscript{\scriptsize{$\pm$0.3}} & 44.27\textsubscript{\scriptsize{$\pm$0.2}} & 32.94\textsubscript{\scriptsize{$\pm$0.6}} & 72.03\textsubscript{\scriptsize{$\pm$0.2}} & 44.67\textsubscript{\scriptsize{$\pm$0.1}} & 13.16\textsubscript{\scriptsize{$\pm$2.8}} & 47.70\textsubscript{\scriptsize{$\pm$0.1}} & 56.23\textsubscript{\scriptsize{$\pm$0.1}} & 35.82 \\
            {\deyoabb}   & 41.97\textsubscript{\scriptsize{$\pm$0.7}} & 44.21\textsubscript{\scriptsize{$\pm$0.4}} & 43.03\textsubscript{\scriptsize{$\pm$0.7}} & 21.81\textsubscript{\scriptsize{$\pm$0.1}} & 23.28\textsubscript{\scriptsize{$\pm$0.3}} & 40.87\textsubscript{\scriptsize{$\pm$0.1}} & 26.10\textsubscript{\scriptsize{$\pm$4.3}} & 53.24\textsubscript{\scriptsize{$\pm$0.2}} & 51.22\textsubscript{\scriptsize{$\pm$0.2}} & 35.13\textsubscript{\scriptsize{$\pm$6.1}} & 72.39\textsubscript{\scriptsize{$\pm$0.1}} & 52.50\textsubscript{\scriptsize{$\pm$0.2}} & 42.44\textsubscript{\scriptsize{$\pm$0.3}} & 59.10\textsubscript{\scriptsize{$\pm$0.0}} & 58.73\textsubscript{\scriptsize{$\pm$0.1}} & 44.40 \\
            \hline
            {\ens}    & 42.54\textsubscript{\scriptsize{$\pm$0.4}} & 44.95\textsubscript{\scriptsize{$\pm$0.3}} & 43.61\textsubscript{\scriptsize{$\pm$0.4}}& 22.47\textsubscript{\scriptsize{$\pm$0.1}} & 23.97\textsubscript{\scriptsize{$\pm$0.2}} & 41.00\textsubscript{\scriptsize{$\pm$0.1}} & 30.90\textsubscript{\scriptsize{$\pm$0.3}} & 53.70\textsubscript{\scriptsize{$\pm$0.2}} & 51.66\textsubscript{\scriptsize{$\pm$0.1}} & 27.93\textsubscript{\scriptsize{$\pm$0.2}} & 73.37\textsubscript{\scriptsize{$\pm$0.3}} & 52.89\textsubscript{\scriptsize{$\pm$0.1}} & 47.01\textsubscript{\scriptsize{$\pm$0.2}} & 59.66\textsubscript{\scriptsize{$\pm$0.0}} & 59.53\textsubscript{\scriptsize{$\pm$0.1}} & 45.01 \\
            {\svgd}   & 42.57\textsubscript{\scriptsize{$\pm$0.3}} & 44.90\textsubscript{\scriptsize{$\pm$0.3}} & 43.79\textsubscript{\scriptsize{$\pm$0.3}} & 22.72\textsubscript{\scriptsize{$\pm$0.1}} & 24.20\textsubscript{\scriptsize{$\pm$0.2}} & 40.89\textsubscript{\scriptsize{$\pm$0.1}} & 32.04\textsubscript{\scriptsize{$\pm$0.3}} & 53.40\textsubscript{\scriptsize{$\pm$0.1}} & 51.69\textsubscript{\scriptsize{$\pm$0.1}} & 28.41\textsubscript{\scriptsize{$\pm$0.1}} & 73.34\textsubscript{\scriptsize{$\pm$0.0}} & 53.00\textsubscript{\scriptsize{$\pm$0.1}} & 46.39\textsubscript{\scriptsize{$\pm$0.1}} & 59.64\textsubscript{\scriptsize{$\pm$0.1}} & 59.62\textsubscript{\scriptsize{$\pm$0.1}}& 45.11 \\
            {\kl}     & 42.50\textsubscript{\scriptsize{$\pm$0.2}} & 44.74\textsubscript{\scriptsize{$\pm$0.2}} & 43.70\textsubscript{\scriptsize{$\pm$0.3}} & 22.50\textsubscript{\scriptsize{$\pm$0.1}} & 24.06\textsubscript{\scriptsize{$\pm$0.2}} & 40.77\textsubscript{\scriptsize{$\pm$0.1}} & 35.05\textsubscript{\scriptsize{$\pm$0.2}} & 53.67\textsubscript{\scriptsize{$\pm$0.1}} & 51.80\textsubscript{\scriptsize{$\pm$0.2}} & 32.44\textsubscript{\scriptsize{$\pm$0.1}} & 73.33\textsubscript{\scriptsize{$\pm$0.0}} & 52.97\textsubscript{\scriptsize{$\pm$0.1}} & 47.02\textsubscript{\scriptsize{$\pm$0.2}} & 59.55\textsubscript{\scriptsize{$\pm$0.1}} & 59.58\textsubscript{\scriptsize{$\pm$0.1}} & \underline{45.58} \\
            \rowcolor{LightCyan}
            \grad    & 42.73\textsubscript{\scriptsize{$\pm$0.2}} & 44.99\textsubscript{\scriptsize{$\pm$0.2}} & 43.88\textsubscript{\scriptsize{$\pm$0.2}} & 22.49\textsubscript{\scriptsize{$\pm$0.1}} & 24.32\textsubscript{\scriptsize{$\pm$0.1}} & 41.07\textsubscript{\scriptsize{$\pm$0.1}} & 25.63\textsubscript{\scriptsize{$\pm$0.2}} & 54.10\textsubscript{\scriptsize{$\pm$0.1}} & 52.03\textsubscript{\scriptsize{$\pm$0.0}} & 58.88\textsubscript{\scriptsize{$\pm$0.1}} & 73.33\textsubscript{\scriptsize{$\pm$0.1}} & 53.28\textsubscript{\scriptsize{$\pm$0.1}} & 47.69\textsubscript{\scriptsize{$\pm$0.2}} & 59.41\textsubscript{\scriptsize{$\pm$0.1}} & 59.70\textsubscript{\scriptsize{$\pm$0.0}} & \textbf{46.90} \\
            \hline
            \hline
            VitBase-LN &9.51 &6.70 &8.21 &28.88 &23.40 &33.91 &27.11 &15.90 &26.48 &47.20 &54.70 &44.11 &30.51 &44.50 &47.81 &29.91 \\
            {Tent}  & 51.15\textsubscript{\scriptsize{$\pm$0.3}} & 51.78\textsubscript{\scriptsize{$\pm$0.2}} & 52.50\textsubscript{\scriptsize{$\pm$0.2}} & 52.18\textsubscript{\scriptsize{$\pm$0.1}} & 47.68\textsubscript{\scriptsize{$\pm$0.2}} & 56.28\textsubscript{\scriptsize{$\pm$0.3}} & 49.06\textsubscript{\scriptsize{$\pm$0.3}} &  8.78\textsubscript{\scriptsize{$\pm$0.4}} & 15.92\textsubscript{\scriptsize{$\pm$0.1}} & 67.25\textsubscript{\scriptsize{$\pm$0.1}} & 73.45\textsubscript{\scriptsize{$\pm$0.2}}& 66.64\textsubscript{\scriptsize{$\pm$0.3}} & 52.09\textsubscript{\scriptsize{$\pm$0.4}} & 64.93\textsubscript{\scriptsize{$\pm$0.2}} & 63.98\textsubscript{\scriptsize{$\pm$0.1}} & 51.58 \\
            {SAR}   & 50.25\textsubscript{\scriptsize{$\pm$0.4}} & 50.65\textsubscript{\scriptsize{$\pm$0.7}} & 51.85\textsubscript{\scriptsize{$\pm$0.3}} & 51.61\textsubscript{\scriptsize{$\pm$0.2}} & 48.92\textsubscript{\scriptsize{$\pm$0.1}} & 56.71\textsubscript{\scriptsize{$\pm$0.1}} & 50.55\textsubscript{\scriptsize{$\pm$0.3}} & 20.45\textsubscript{\scriptsize{$\pm$0.4}} & 54.45\textsubscript{\scriptsize{$\pm$0.6}} & 67.42\textsubscript{\scriptsize{$\pm$0.9}} & 74.92\textsubscript{\scriptsize{$\pm$0.7}} & 65.84\textsubscript{\scriptsize{$\pm$0.0}} & 54.68\textsubscript{\scriptsize{$\pm$0.1}} & 66.57\textsubscript{\scriptsize{$\pm$0.1}} & 64.91\textsubscript{\scriptsize{$\pm$0.1}} & 55.31 \\
            {\deyoabb}  & 52.61\textsubscript{\scriptsize{$\pm$0.7}} & 53.50\textsubscript{\scriptsize{$\pm$1.7}} & 53.46\textsubscript{\scriptsize{$\pm$0.8}} & 54.81\textsubscript{\scriptsize{$\pm$0.1}} & 55.42\textsubscript{\scriptsize{$\pm$0.1}} & 61.92\textsubscript{\scriptsize{$\pm$0.1}} & 38.37\textsubscript{\scriptsize{$\pm$4.8}} & 64.55\textsubscript{\scriptsize{$\pm$0.1}} & 62.93\textsubscript{\scriptsize{$\pm$0.0}} & 70.91\textsubscript{\scriptsize{$\pm$0.1}} & 76.19\textsubscript{\scriptsize{$\pm$0.0}} & 60.14\textsubscript{\scriptsize{$\pm$0.1}} & 65.17\textsubscript{\scriptsize{$\pm$0.1}} & 71.00\textsubscript{\scriptsize{$\pm$0.1}} & 67.56\textsubscript{\scriptsize{$\pm$0.3}} & 60.57 \\
            \hline
            {\ens}   & 53.47\textsubscript{\scriptsize{$\pm$0.3}} & 54.67\textsubscript{\scriptsize{$\pm$0.2}} & 54.80\textsubscript{\scriptsize{$\pm$0.4}} & 56.17\textsubscript{\scriptsize{$\pm$0.2}} & 56.51\textsubscript{\scriptsize{$\pm$0.1}} & 62.18\textsubscript{\scriptsize{$\pm$0.1}} & 47.20\textsubscript{\scriptsize{$\pm$0.3}} & 64.91\textsubscript{\scriptsize{$\pm$0.3}} & 63.04\textsubscript{\scriptsize{$\pm$0.1}} & 70.97\textsubscript{\scriptsize{$\pm$0.1}} & 76.64\textsubscript{\scriptsize{$\pm$0.1}} & 66.20\textsubscript{\scriptsize{$\pm$0.1}} & 66.10\textsubscript{\scriptsize{$\pm$0.1}} & 71.42\textsubscript{\scriptsize{$\pm$0.1}} & 68.29\textsubscript{\scriptsize{$\pm$0.3}} & 62.17 \\
            {\svgd}  & 54.11\textsubscript{\scriptsize{$\pm$0.3}} & 54.78\textsubscript{\scriptsize{$\pm$0.2}} & 54.55\textsubscript{\scriptsize{$\pm$0.3}} & 57.01\textsubscript{\scriptsize{$\pm$0.2}} & 57.08\textsubscript{\scriptsize{$\pm$0.1}} & 62.85\textsubscript{\scriptsize{$\pm$0.1}} & 51.33\textsubscript{\scriptsize{$\pm$0.2}} & 65.01\textsubscript{\scriptsize{$\pm$0.2}} & 63.69\textsubscript{\scriptsize{$\pm$0.1}} & 72.04\textsubscript{\scriptsize{$\pm$0.1}} & 77.18\textsubscript{\scriptsize{$\pm$0.1}} & 68.04\textsubscript{\scriptsize{$\pm$0.1}} & 66.22\textsubscript{\scriptsize{$\pm$0.2}} & 71.73\textsubscript{\scriptsize{$\pm$0.1}} & 68.53\textsubscript{\scriptsize{$\pm$0.2}} & 62.94 \\
            {\kl}    & 54.83\textsubscript{\scriptsize{$\pm$0.3}} & 55.78\textsubscript{\scriptsize{$\pm$0.3}} & 55.97\textsubscript{\scriptsize{$\pm$0.2}} & 56.46\textsubscript{\scriptsize{$\pm$0.3}} & 56.84\textsubscript{\scriptsize{$\pm$0.2}} & 62.66\textsubscript{\scriptsize{$\pm$0.2}} & 54.78\textsubscript{\scriptsize{$\pm$0.2}} & 65.05\textsubscript{\scriptsize{$\pm$0.1}} & 63.69\textsubscript{\scriptsize{$\pm$0.2}} & 72.30\textsubscript{\scriptsize{$\pm$0.2}} & 77.24\textsubscript{\scriptsize{$\pm$0.1}} & 68.04\textsubscript{\scriptsize{$\pm$0.1}} & 66.00\textsubscript{\scriptsize{$\pm$0.0}} & 71.82\textsubscript{\scriptsize{$\pm$0.1}} & 68.89\textsubscript{\scriptsize{$\pm$0.2}} & \textbf{63.36} \\
            \rowcolor{LightCyan}
            \grad   & 54.83\textsubscript{\scriptsize{$\pm$0.3}} & 55.65\textsubscript{\scriptsize{$\pm$0.1}} & 56.13\textsubscript{\scriptsize{$\pm$0.2}} & 56.68\textsubscript{\scriptsize{$\pm$0.3}} & 57.13\textsubscript{\scriptsize{$\pm$0.1}} & 62.86\textsubscript{\scriptsize{$\pm$0.1}} & 50.72\textsubscript{\scriptsize{$\pm$0.1}} & 64.98\textsubscript{\scriptsize{$\pm$0.1}} & 63.48\textsubscript{\scriptsize{$\pm$0.0}} & 72.51\textsubscript{\scriptsize{$\pm$0.1}} & 77.34\textsubscript{\scriptsize{$\pm$0.1}} & 67.93\textsubscript{\scriptsize{$\pm$0.1}} & 66.01\textsubscript{\scriptsize{$\pm$0.0}} & 72.20\textsubscript{\scriptsize{$\pm$0.1}} & 68.76\textsubscript{\scriptsize{$\pm$0.0}} & \underline{63.14} \\
            \hline
        \end{tabular}
    }

    \smallskip

    \setlength{\tabcolsep}{1.6pt}
    \scalebox{0.485}{
        \begin{tabular}{l|ccc|cccc|cccc|cccc|c}
            \multicolumn{1}{c}{} & \multicolumn{3}{c}{Noise} & \multicolumn{4}{c}{Blur} & \multicolumn{4}{c}{Weather} & \multicolumn{4}{c}{Digital} & \\
            \textbf{Label Shifts} & \textbf{Gauss.} & \textbf{Shot} & \textbf{Impl.} & \textbf{Defoc.} & \textbf{Glass} & \textbf{Motion} & \textbf{Zoom} & \textbf{Snow} & \textbf{Frost} & \textbf{Fog} & \textbf{Brit.} & \textbf{Contr.} & \textbf{Elastic} & \textbf{Pixel} & \textbf{JPEG} & \textbf{Avg.} \\
            \hline \hline
            ResNet-50-GN & 17.9& 19.9&17.9 &19.7 &11.3 &21.3 &24.9 &40.4 &47.4 &33.6 &69.2 &36.3 &18.7 &28.4 &52.2 &30.6 \\
            {Tent} & 1.67\textsubscript{\scriptsize{$\pm$0.2}} & 2.02\textsubscript{\scriptsize{$\pm$0.2}} & 2.01\textsubscript{\scriptsize{$\pm$0.1}} & 12.17\textsubscript{\scriptsize{$\pm$0.3}} & 2.14\textsubscript{\scriptsize{$\pm$0.3}} & 23.00\textsubscript{\scriptsize{$\pm$0.2}} & 13.74\textsubscript{\scriptsize{$\pm$0.4}} & 8.86\textsubscript{\scriptsize{$\pm$0.2}} & 13.69\textsubscript{\scriptsize{$\pm$0.2}} & 1.00\textsubscript{\scriptsize{$\pm$0.1}} & 73.00\textsubscript{\scriptsize{$\pm$0.3}} & 48.44\textsubscript{\scriptsize{$\pm$0.3}} & 3.51\textsubscript{\scriptsize{$\pm$0.2}} & 54.90\textsubscript{\scriptsize{$\pm$0.2}} & 57.18\textsubscript{\scriptsize{$\pm$0.3}} & 21.15 \\
            {SAR} & 28.71\textsubscript{\scriptsize{$\pm$0.1}} & 31.19\textsubscript{\scriptsize{$\pm$0.4}} & 29.39\textsubscript{\scriptsize{$\pm$0.9}} & 18.52\textsubscript{\scriptsize{$\pm$0.4}} & 15.10\textsubscript{\scriptsize{$\pm$1.2}} & 28.71\textsubscript{\scriptsize{$\pm$0.5}} & 29.84\textsubscript{\scriptsize{$\pm$3.5}} & 42.86\textsubscript{\scriptsize{$\pm$2.6}} & 44.33\textsubscript{\scriptsize{$\pm$0.4}} & 35.32\textsubscript{\scriptsize{$\pm$0.9}} & 72.07\textsubscript{\scriptsize{$\pm$0.1}} & 44.65\textsubscript{\scriptsize{$\pm$0.2}} & 14.40\textsubscript{\scriptsize{$\pm$1.9}} & 47.04\textsubscript{\scriptsize{$\pm$0.5}} & 56.22\textsubscript{\scriptsize{$\pm$0.1}} & 35.89 \\
            {\deyoabb} & 40.27\textsubscript{\scriptsize{$\pm$0.5}} & 43.58\textsubscript{\scriptsize{$\pm$0.2}} & 41.75\textsubscript{\scriptsize{$\pm$0.3}}& 22.42\textsubscript{\scriptsize{$\pm$0.0}} & 22.69\textsubscript{\scriptsize{$\pm$0.2}} & 40.93\textsubscript{\scriptsize{$\pm$0.2}} & 14.63\textsubscript{\scriptsize{$\pm$2.3}} & 53.19\textsubscript{\scriptsize{$\pm$0.4}} & 51.78\textsubscript{\scriptsize{$\pm$0.5}} & 1.03\textsubscript{\scriptsize{$\pm$0.2}} & 72.89\textsubscript{\scriptsize{$\pm$0.1}} & 52.93\textsubscript{\scriptsize{$\pm$0.2}} & 47.81\textsubscript{\scriptsize{$\pm$1.2}} & 59.48\textsubscript{\scriptsize{$\pm$0.0}} & 59.30\textsubscript{\scriptsize{$\pm$0.1}} & 41.65 \\
            \hline
            {\ens} & 41.37\textsubscript{\scriptsize{$\pm$0.3}} & 44.24\textsubscript{\scriptsize{$\pm$0.2}} & 42.43\textsubscript{\scriptsize{$\pm$0.3}} & 22.52\textsubscript{\scriptsize{$\pm$0.1}} & 23.21\textsubscript{\scriptsize{$\pm$0.2}} & 40.95\textsubscript{\scriptsize{$\pm$0.2}} & 12.46\textsubscript{\scriptsize{$\pm$0.4}} & 52.49\textsubscript{\scriptsize{$\pm$0.3}} & 51.83\textsubscript{\scriptsize{$\pm$0.4}} & 1.39\textsubscript{\scriptsize{$\pm$0.4}} & 73.30\textsubscript{\scriptsize{$\pm$0.3}} & 52.91\textsubscript{\scriptsize{$\pm$0.2}} & 47.81\textsubscript{\scriptsize{$\pm$0.3}} & 59.53\textsubscript{\scriptsize{$\pm$0.1}} & 59.59\textsubscript{\scriptsize{$\pm$0.2}} & 41.74 \\
            {\svgd} & 42.43\textsubscript{\scriptsize{$\pm$0.3}} & 44.82\textsubscript{\scriptsize{$\pm$0.3}} & 43.66\textsubscript{\scriptsize{$\pm$0.3}} & 22.98\textsubscript{\scriptsize{$\pm$0.1}} & 23.54\textsubscript{\scriptsize{$\pm$0.2}} & 40.74\textsubscript{\scriptsize{$\pm$0.2}} & 17.55\textsubscript{\scriptsize{$\pm$0.3}} & 53.11\textsubscript{\scriptsize{$\pm$0.3}} & 51.97\textsubscript{\scriptsize{$\pm$0.2}} & 4.31\textsubscript{\scriptsize{$\pm$0.3}} & 73.51\textsubscript{\scriptsize{$\pm$0.3}} & 52.97\textsubscript{\scriptsize{$\pm$0.2}} & 46.82\textsubscript{\scriptsize{$\pm$0.2}} & 59.38\textsubscript{\scriptsize{$\pm$0.1}} & 59.58\textsubscript{\scriptsize{$\pm$0.2}} & \underline{42.49} \\
            {\kl} & 42.49\textsubscript{\scriptsize{$\pm$0.3}} & 44.87\textsubscript{\scriptsize{$\pm$0.3}} & 43.56\textsubscript{\scriptsize{$\pm$0.2}} & 22.92\textsubscript{\scriptsize{$\pm$0.1}} & 23.70\textsubscript{\scriptsize{$\pm$0.2}} & 40.76\textsubscript{\scriptsize{$\pm$0.2}} & 17.59\textsubscript{\scriptsize{$\pm$0.2}} & 53.07\textsubscript{\scriptsize{$\pm$0.3}} & 52.03\textsubscript{\scriptsize{$\pm$0.2}} & 1.87\textsubscript{\scriptsize{$\pm$0.3}} & 73.38\textsubscript{\scriptsize{$\pm$0.3}} & 52.93\textsubscript{\scriptsize{$\pm$0.2}} & 46.28\textsubscript{\scriptsize{$\pm$0.1}} & 59.18\textsubscript{\scriptsize{$\pm$0.2}} & 59.28\textsubscript{\scriptsize{$\pm$0.1}} & 42.26 \\
            \rowcolor{LightCyan}
            \grad  & 42.42\textsubscript{\scriptsize{$\pm$0.2}} & 44.80\textsubscript{\scriptsize{$\pm$0.2}} & 43.64\textsubscript{\scriptsize{$\pm$0.2}} & 22.95\textsubscript{\scriptsize{$\pm$0.1}} & 23.59\textsubscript{\scriptsize{$\pm$0.2}} & 40.79\textsubscript{\scriptsize{$\pm$0.1}} & 17.53\textsubscript{\scriptsize{$\pm$0.1}} & 53.09\textsubscript{\scriptsize{$\pm$0.1}} & 51.97\textsubscript{\scriptsize{$\pm$0.0}} & 15.79\textsubscript{\scriptsize{$\pm$0.1}} & 73.50\textsubscript{\scriptsize{$\pm$0.1}} & 53.00\textsubscript{\scriptsize{$\pm$0.2}} & 46.92\textsubscript{\scriptsize{$\pm$0.0}} & 59.41\textsubscript{\scriptsize{$\pm$0.1}} & 59.61\textsubscript{\scriptsize{$\pm$0.0}} & \textbf{43.27} \\
            \hline
            \hline
            ViTBase-LN & 9.42& 6.73& 8.32&29.11 &23.41 &34.03 &27.01 &15.80 &26.30 &47.41 &54.72 &43.90 &30.50 &44.5 &47.6 &29.9 \\
            {Tent} & 53.41\textsubscript{\scriptsize{$\pm$0.2}} & 53.16\textsubscript{\scriptsize{$\pm$0.2}} & 54.36\textsubscript{\scriptsize{$\pm$0.3}} & 54.34\textsubscript{\scriptsize{$\pm$0.4}} & 52.46\textsubscript{\scriptsize{$\pm$0.4}} & 58.74\textsubscript{\scriptsize{$\pm$0.6}} & 52.83\textsubscript{\scriptsize{$\pm$0.2}} & 4.29\textsubscript{\scriptsize{$\pm$0.2}} & 7.04\textsubscript{\scriptsize{$\pm$0.1}} & 69.46\textsubscript{\scriptsize{$\pm$0.3}} & 74.92\textsubscript{\scriptsize{$\pm$0.3}} & 67.42\textsubscript{\scriptsize{$\pm$0.4}} & 59.97\textsubscript{\scriptsize{$\pm$0.2}} & 67.89\textsubscript{\scriptsize{$\pm$0.2}} & 66.13\textsubscript{\scriptsize{$\pm$0.1}} & 53.10 \\
            {SAR} & 51.90\textsubscript{\scriptsize{$\pm$0.4}} & 51.47\textsubscript{\scriptsize{$\pm$0.7}} & 53.02\textsubscript{\scriptsize{$\pm$0.3}} & 51.69\textsubscript{\scriptsize{$\pm$0.2}} & 48.57\textsubscript{\scriptsize{$\pm$0.1}}& 56.55\textsubscript{\scriptsize{$\pm$0.1}} & 50.62\textsubscript{\scriptsize{$\pm$0.3}} & 26.45\textsubscript{\scriptsize{$\pm$0.4}} & 54.29\textsubscript{\scriptsize{$\pm$0.6}} & 68.00\textsubscript{\scriptsize{$\pm$0.9}} & 74.28\textsubscript{\scriptsize{$\pm$0.7}} & 65.65\textsubscript{\scriptsize{$\pm$0.0}} & 55.44\textsubscript{\scriptsize{$\pm$0.1}} & 66.68\textsubscript{\scriptsize{$\pm$0.1}} & 65.07\textsubscript{\scriptsize{$\pm$0.0}} & 55.98 \\
            {\deyoabb} & 53.92\textsubscript{\scriptsize{$\pm$0.3}} & 54.86\textsubscript{\scriptsize{$\pm$0.7}} & 55.08\textsubscript{\scriptsize{$\pm$0.4}} & 54.78\textsubscript{\scriptsize{$\pm$0.1}} & 55.70\textsubscript{\scriptsize{$\pm$0.2}} & 61.63\textsubscript{\scriptsize{$\pm$0.1}} & 53.34\textsubscript{\scriptsize{$\pm$0.2}} & 64.31\textsubscript{\scriptsize{$\pm$0.4}}& 63.10\textsubscript{\scriptsize{$\pm$1.1}} & 70.75\textsubscript{\scriptsize{$\pm$0.3}} & 76.81\textsubscript{\scriptsize{$\pm$0.1}} & 65.90\textsubscript{\scriptsize{$\pm$0.3}}& 65.60\textsubscript{\scriptsize{$\pm$0.3}}& 71.65\textsubscript{\scriptsize{$\pm$0.2}} & 68.10\textsubscript{\scriptsize{$\pm$0.2}} & 62.37 \\
            \hline
            {\ens} & 54.11\textsubscript{\scriptsize{$\pm$0.3}} & 54.91\textsubscript{\scriptsize{$\pm$0.4}} & 55.41\textsubscript{\scriptsize{$\pm$0.4}} & 55.75\textsubscript{\scriptsize{$\pm$0.1}} & 56.25\textsubscript{\scriptsize{$\pm$0.2}} & 61.95\textsubscript{\scriptsize{$\pm$0.1}} & 56.15\textsubscript{\scriptsize{$\pm$0.1}} & 64.14\textsubscript{\scriptsize{$\pm$0.1}} & 63.03\textsubscript{\scriptsize{$\pm$0.3}} & 71.76\textsubscript{\scriptsize{$\pm$0.4}} & 76.74\textsubscript{\scriptsize{$\pm$0.2}} & 67.47\textsubscript{\scriptsize{$\pm$0.2}} & 65.54\textsubscript{\scriptsize{$\pm$0.2}} & 71.51\textsubscript{\scriptsize{$\pm$0.1}} & 67.95\textsubscript{\scriptsize{$\pm$0.2}} & 62.84 \\
            {\svgd} & 54.59\textsubscript{\scriptsize{$\pm$0.3}} & 55.41\textsubscript{\scriptsize{$\pm$0.3}} & 55.94\textsubscript{\scriptsize{$\pm$0.3}} & 56.17\textsubscript{\scriptsize{$\pm$0.1}} & 56.77\textsubscript{\scriptsize{$\pm$0.2}} & 62.50\textsubscript{\scriptsize{$\pm$0.1}} & 56.22\textsubscript{\scriptsize{$\pm$0.1}} & 64.63\textsubscript{\scriptsize{$\pm$0.2}} & 63.46\textsubscript{\scriptsize{$\pm$0.3}} & 72.32\textsubscript{\scriptsize{$\pm$0.3}} & 77.25\textsubscript{\scriptsize{$\pm$0.2}} & 67.54\textsubscript{\scriptsize{$\pm$0.2}} & 65.98\textsubscript{\scriptsize{$\pm$0.2}} & 71.98\textsubscript{\scriptsize{$\pm$0.1}} & 68.55\textsubscript{\scriptsize{$\pm$0.1}} & \underline{63.29} \\
            {\kl} & 54.56\textsubscript{\scriptsize{$\pm$0.2}} & 55.33\textsubscript{\scriptsize{$\pm$0.2}} & 55.89\textsubscript{\scriptsize{$\pm$0.2}} & 56.08\textsubscript{\scriptsize{$\pm$0.1}} & 56.52\textsubscript{\scriptsize{$\pm$0.1}} & 62.47\textsubscript{\scriptsize{$\pm$0.1}} & 56.76\textsubscript{\scriptsize{$\pm$0.1}} & 64.65\textsubscript{\scriptsize{$\pm$0.2}} & 63.57\textsubscript{\scriptsize{$\pm$0.2}} & 72.21\textsubscript{\scriptsize{$\pm$0.2}} & 77.24\textsubscript{\scriptsize{$\pm$0.1}} & 67.97\textsubscript{\scriptsize{$\pm$0.1}} & 65.81\textsubscript{\scriptsize{$\pm$0.1}} & 71.98\textsubscript{\scriptsize{$\pm$0.1}} & 68.35\textsubscript{\scriptsize{$\pm$0.1}} & 62.24 \\
            \rowcolor{LightCyan}
            \grad  & 54.52\textsubscript{\scriptsize{$\pm$0.2}} & 55.37\textsubscript{\scriptsize{$\pm$0.1}} & 55.89\textsubscript{\scriptsize{$\pm$0.1}} & 56.09\textsubscript{\scriptsize{$\pm$0.1}} & 56.74\textsubscript{\scriptsize{$\pm$0.0}} & 62.47\textsubscript{\scriptsize{$\pm$0.1}} & 57.31\textsubscript{\scriptsize{$\pm$0.1}} & 64.63\textsubscript{\scriptsize{$\pm$0.0}} & 63.55\textsubscript{\scriptsize{$\pm$0.2}} & 72.36\textsubscript{\scriptsize{$\pm$0.1}} & 77.23\textsubscript{\scriptsize{$\pm$0.1}} & 67.80\textsubscript{\scriptsize{$\pm$0.0}}& 65.92\textsubscript{\scriptsize{$\pm$0.1}} & 72.01\textsubscript{\scriptsize{$\pm$0.0}} & 68.52\textsubscript{\scriptsize{$\pm$0.1}} & \textbf{63.36} \\
            \hline
        \end{tabular}
    }
\end{table*}
\noindent\textbf{Comparison on Wild Scenario.}
We evaluate our population-based adaptation framework on ImageNet-C under three challenging wild settings: (1) batch size $1$, (2) label distribution shift, and (3) mixed distribution shifts. Results are reported in~\cref{tab:imagenet-c,tab:imagenet-c-mix-shift}.

\noindent\underline{1) Batch size 1.}
When adaptation is performed on a single test sample, entropy minimization becomes highly unstable due to the lack of batch statistics. 
Table~\ref{tab:imagenet-c} shows that simply maintaining multiple normalization particles (\texttt{Ens}) already improves robustness over single-model adaptation. 
For ViTBase-LN, the ensemble setting improves average accuracy by 2.23\% over DeYO, outperforming all baselines across most corruption types. 
ResNet-50-GN exhibits similar gains.

Introducing explicit diversification further strengthens performance. 
Among all strategies, gradient-based diversification (\texttt{Grad}) achieves the best results, improving average accuracy by 3.21\% for ViTBase-LN and 2.50\% for ResNet-50-GN over DeYO. 
Notably, under severe corruptions such as Fog, \texttt{Grad} attains the highest overall accuracy (58.88\%), indicating improved stability under extreme underspecification.

\noindent\underline{2) Label Distribution Shift.}
Under infinite class imbalance ratio~\cite{niu2023towards}, entropy minimization is prone to reinforcing dominant classes.
The ensemble baseline (\texttt{Ens}) improves over DeYO by 0.47\% on ViTBase-LN, suggesting that population-based adaptation mitigates collapse toward biased modes.
Applying diversification further enhances robustness: \texttt{Grad} achieves improvements of 0.99\% (ViTBase-LN) and 1.62\% (ResNet-50-GN) over DeYO.
These gains indicate that gradient-level functional repulsion reduces overconfidence under skewed label distributions.

\noindent\underline{3) Mixed Distribution Shifts.}
We further evaluate performance on mixtures of 15 corruption types at severity level 5.
As shown in~\cref{tab:imagenet-c-mix-shift}, gradient-based diversification consistently delivers the strongest robustness.
Compared with the best-performing baseline, \texttt{Grad} improves accuracy by 3\% on ResNet-50-GN and 4.32\% on ViTBase-LN.
These results confirm that particle-level exploration combined with functional repulsion effectively stabilizes adaptation under compound shifts.

\begin{table}[t]
\caption{Comparisons with baselines on ImageNet-C at severity 5 under a mixture of 15 corruptions. Results report top-1 and top-5 accuracy (\%).}
\label{tab:imagenet-c-mix-shift}
\centering

\begin{minipage}{0.48\linewidth}
\centering
\small
\scalebox{0.6}{
\begin{tabular}{l|cc}
\toprule
\multicolumn{3}{c}{\textbf{ResNet-50-GN}} \\
\midrule
Mix Shifts & top@1 & top@5 \\
\midrule
No Adapt & 30.61 & 53.89 \\
Tent & 29.80$_{\pm0.2}$ & 30.65$_{\pm0.2}$ \\
SAR & 38.12$_{\pm0.1}$ & 59.70$_{\pm0.1}$ \\
\deyoabb & 31.36$_{\pm1.3}$ & 50.78$_{\pm0.1}$ \\
\midrule
\ens & \underline{33.06}$_{\pm0.3}$ & \underline{55.91}$_{\pm0.1}$ \\
\svgd & 32.71$_{\pm0.4}$ & 55.79$_{\pm0.2}$ \\
\kl & 32.66$_{\pm0.3}$ & 52.23$_{\pm0.2}$ \\
\rowcolor{LightCyan}
\grad & \textbf{34.36}$_{\pm0.2}$ & \textbf{56.21}$_{\pm0.1}$ \\
\bottomrule
\end{tabular}}
\end{minipage}
\hfill
\begin{minipage}{0.48\linewidth}
\centering
\small
\scalebox{0.6}{
\begin{tabular}{l|cc}
\toprule
\multicolumn{3}{c}{\textbf{ViTBase-LN}} \\
\midrule
Mix Shifts & top@1 & top@5 \\
\midrule
No Adapt & 29.94 & 54.21 \\
Tent & 32.36$_{\pm0.3}$ & 47.40$_{\pm0.4}$ \\
SAR & 57.78$_{\pm0.1}$ & 78.14$_{\pm0.1}$ \\
\deyoabb & 57.05$_{\pm1.1}$ & 78.35$_{\pm1.3}$ \\
\midrule
\ens & 60.56$_{\pm0.4}$ & 79.23$_{\pm0.4}$ \\
\svgd & \underline{60.84}$_{\pm0.3}$ & 80.30$_{\pm0.2}$ \\
\kl & 60.37$_{\pm0.3}$ & \underline{80.42}$_{\pm0.3}$ \\
\rowcolor{LightCyan}
\grad & \textbf{61.36}$_{\pm0.2}$ & \textbf{81.45}$_{\pm0.2}$ \\
\bottomrule
\end{tabular}}
\end{minipage}

\end{table}

Overall, results across all wild scenarios demonstrate that (i) maintaining multiple adaptation particles already improves stability over single-model entropy minimization, and (ii) structured diversification, particularly gradient-based repulsion, provides additional robustness gains by preventing trajectory collapse under underspecified adaptation.

\begin{table*}[!t]
    \caption{Accuracy (\%) comparison with baseline methods, averaged over five random seeds, on ImageNet-C wild scenarios (severity level 5) using batch size 1 and three particles, each employing a different optimizer.}
    \label{tab:noAug}
    \centering
    \setlength{\tabcolsep}{1.6pt}
    \scalebox{0.48}{
        \begin{tabular}{l|ccc|cccc|cccc|cccc|c}
            \multicolumn{1}{c}{} & \multicolumn{3}{c}{Noise} & \multicolumn{4}{c}{Blur} & \multicolumn{4}{c}{Weather} & \multicolumn{4}{c}{Digital} & \\
            \textbf{Batch Size 1} & \textbf{Gauss.} & \textbf{Shot} & \textbf{Impl.} & \textbf{Defoc.} & \textbf{Glass} & \textbf{Motion} & \textbf{Zoom} & \textbf{Snow} & \textbf{Frost} & \textbf{Fog} & \textbf{Brit.} & \textbf{Contr.} & \textbf{Elastic} & \textbf{Pixel} & \textbf{JPEG} & \textbf{Avg.} \\
            \hline \hline     
            VitBase-LN &9.51 &6.70 &8.21 &28.88 &23.40 &33.91 &27.11 &15.90 &26.48 &47.20 &54.70 &44.11 &30.51 &44.50 &47.81 &29.91 \\
            {Tent}  & 51.15\textsubscript{\scriptsize{$\pm$0.3}} & 51.78\textsubscript{\scriptsize{$\pm$0.2}} & 52.50\textsubscript{\scriptsize{$\pm$0.2}} & 52.18\textsubscript{\scriptsize{$\pm$0.1}} & 47.68\textsubscript{\scriptsize{$\pm$0.2}} & 56.28\textsubscript{\scriptsize{$\pm$0.3}} & 49.06\textsubscript{\scriptsize{$\pm$0.3}} &  8.78\textsubscript{\scriptsize{$\pm$0.4}} & 15.92\textsubscript{\scriptsize{$\pm$0.1}} & 67.25\textsubscript{\scriptsize{$\pm$0.1}} & 73.45\textsubscript{\scriptsize{$\pm$0.2}}& 66.64\textsubscript{\scriptsize{$\pm$0.3}} & 52.09\textsubscript{\scriptsize{$\pm$0.4}} & 64.93\textsubscript{\scriptsize{$\pm$0.2}} & 63.98\textsubscript{\scriptsize{$\pm$0.1}} & 51.58 \\
            {SAR}   & 50.25\textsubscript{\scriptsize{$\pm$0.4}} & 50.65\textsubscript{\scriptsize{$\pm$0.7}} & 51.85\textsubscript{\scriptsize{$\pm$0.3}} & 51.61\textsubscript{\scriptsize{$\pm$0.2}} & 48.92\textsubscript{\scriptsize{$\pm$0.1}} & 56.71\textsubscript{\scriptsize{$\pm$0.1}} & 50.55\textsubscript{\scriptsize{$\pm$0.3}} & 20.45\textsubscript{\scriptsize{$\pm$0.4}} & 54.45\textsubscript{\scriptsize{$\pm$0.6}} & 67.42\textsubscript{\scriptsize{$\pm$0.9}} & 74.92\textsubscript{\scriptsize{$\pm$0.7}} & 65.84\textsubscript{\scriptsize{$\pm$0.0}} & 54.68\textsubscript{\scriptsize{$\pm$0.1}} & 66.57\textsubscript{\scriptsize{$\pm$0.1}} & 64.91\textsubscript{\scriptsize{$\pm$0.1}} & 55.31 \\
            {\deyoabb}  & 52.61\textsubscript{\scriptsize{$\pm$0.7}} & 53.50\textsubscript{\scriptsize{$\pm$1.7}} & 53.46\textsubscript{\scriptsize{$\pm$0.8}} & 54.81\textsubscript{\scriptsize{$\pm$0.1}} & 55.42\textsubscript{\scriptsize{$\pm$0.1}} & 61.92\textsubscript{\scriptsize{$\pm$0.1}} & 38.37\textsubscript{\scriptsize{$\pm$4.8}} & 64.55\textsubscript{\scriptsize{$\pm$0.1}} & 62.93\textsubscript{\scriptsize{$\pm$0.0}} & 70.91\textsubscript{\scriptsize{$\pm$0.1}} & 76.19\textsubscript{\scriptsize{$\pm$0.0}} & 60.14\textsubscript{\scriptsize{$\pm$0.1}} & 65.17\textsubscript{\scriptsize{$\pm$0.1}} & 71.00\textsubscript{\scriptsize{$\pm$0.1}} & 67.56\textsubscript{\scriptsize{$\pm$0.3}} & 60.57 \\
            \hline
            {Naive}   & 52.95\textsubscript{\scriptsize{$\pm$0.5}} & 54.01\textsubscript{\scriptsize{$\pm$0.7}} & 54.23\textsubscript{\scriptsize{$\pm$0.7}}& 55.21\textsubscript{\scriptsize{$\pm$0.1}} & 54.17\textsubscript{\scriptsize{$\pm$0.1}} & 60.43\textsubscript{\scriptsize{$\pm$0.1}} & 55.86\textsubscript{\scriptsize{$\pm$0.9}} & 62.15\textsubscript{\scriptsize{$\pm$0.1}} & 61.23\textsubscript{\scriptsize{$\pm$0.1}} & 71.28\textsubscript{\scriptsize{$\pm$0.1}} & 76.47\textsubscript{\scriptsize{$\pm$0.1}} & 67.21\textsubscript{\scriptsize{$\pm$0.1}} & 64.36\textsubscript{\scriptsize{$\pm$0.0}} & 71.04\textsubscript{\scriptsize{$\pm$0.1}} & 67.63\textsubscript{\scriptsize{$\pm$0.2}} & 61.88 \\
            {\svgd}  & 53.87\textsubscript{\scriptsize{$\pm$0.5}} & 54.65\textsubscript{\scriptsize{$\pm$0.6}} & 54.93\textsubscript{\scriptsize{$\pm$0.5}} & 56.57\textsubscript{\scriptsize{$\pm$0.1}} & 56.24\textsubscript{\scriptsize{$\pm$0.1}} & 62.43\textsubscript{\scriptsize{$\pm$0.1}} & 54.21\textsubscript{\scriptsize{$\pm$0.7}} & 63.72\textsubscript{\scriptsize{$\pm$0.1}} & 62.94\textsubscript{\scriptsize{$\pm$0.1}} & 72.33\textsubscript{\scriptsize{$\pm$0.1}} & 76.82\textsubscript{\scriptsize{$\pm$0.1}} & 68.64\textsubscript{\scriptsize{$\pm$0.1}} & 65.33\textsubscript{\scriptsize{$\pm$0.0}} & 71.26\textsubscript{\scriptsize{$\pm$0.0}} & 68.22\textsubscript{\scriptsize{$\pm$0.2}} & 62.81 \\
            {\kl}    & 53.90\textsubscript{\scriptsize{$\pm$0.5}} & 54.61\textsubscript{\scriptsize{$\pm$0.4}} & 55.15\textsubscript{\scriptsize{$\pm$0.4}} & 56.62\textsubscript{\scriptsize{$\pm$0.1}} & 55.69\textsubscript{\scriptsize{$\pm$0.1}} & 62.13\textsubscript{\scriptsize{$\pm$0.1}} & 57.27\textsubscript{\scriptsize{$\pm$0.4}} & 63.58\textsubscript{\scriptsize{$\pm$0.1}} & 62.45\textsubscript{\scriptsize{$\pm$0.1}} & 72.31\textsubscript{\scriptsize{$\pm$0.1}} & 76.86\textsubscript{\scriptsize{$\pm$0.1}} & 68.49\textsubscript{\scriptsize{$\pm$0.1}} & 64.92\textsubscript{\scriptsize{$\pm$0.1}} & 71.28\textsubscript{\scriptsize{$\pm$0.1}} & 68.43\textsubscript{\scriptsize{$\pm$0.1}} & \underline{62.91} \\
            \rowcolor{LightCyan}
            \grad   & 55.13\textsubscript{\scriptsize{$\pm$0.4}} & 55.28\textsubscript{\scriptsize{$\pm$0.4}} & 56.22\textsubscript{\scriptsize{$\pm$0.2}} & 56.23\textsubscript{\scriptsize{$\pm$0.0}} & 55.84\textsubscript{\scriptsize{$\pm$0.1}} & 62.02\textsubscript{\scriptsize{$\pm$0.0}} & 56.94\textsubscript{\scriptsize{$\pm$0.4}} & 63.73\textsubscript{\scriptsize{$\pm$0.1}} & 62.24\textsubscript{\scriptsize{$\pm$0.0}} & 72.31\textsubscript{\scriptsize{$\pm$0.0}} & 76.90\textsubscript{\scriptsize{$\pm$0.1}} & 68.78\textsubscript{\scriptsize{$\pm$0.0}} & 64.78\textsubscript{\scriptsize{$\pm$0.1}} & 71.37\textsubscript{\scriptsize{$\pm$0.1}} & 68.47\textsubscript{\scriptsize{$\pm$0.0}} & \textbf{63.08} \\
            \hline
        \end{tabular}
    }
\end{table*}

\noindent \textbf{3.2.3 Optimizer Type Diversification:} Beyond explicit regularization, we introduce trajectory-level diversification by assigning different optimization algorithms to different particles. 
While all particles minimize the same entropy objective, the choice of optimizer affects the geometry of parameter updates due to differences in momentum accumulation, adaptive scaling, and weight decay mechanisms. Specifically, for a particle $\theta_i$, the update at iteration $t$ is given by
\begin{equation}
\theta_i^{(t)} = \mathrm{Update}_{\text{opt}_i}
\left(
\theta_i^{(t-1)}, \nabla_{\theta_i} \ell(X; \theta_i^{(t-1)})
\right),
\end{equation}
where $\mathrm{Update}_{\text{opt}_i}$ denotes the update rule defined by the optimizer assigned to particle $i$ (SGD, Adam, or AdamW), including its internal state variables. Using heterogeneous optimizers induces different curvature sensitivities and step dynamics across particles. For example, SGD follows uniform gradient directions, whereas Adam rescales updates according to estimated second-order statistics, and AdamW decouples weight decay from gradient updates. Consequently, even under identical entropy gradients, particles follow distinct trajectories in parameter space. This optimizer heterogeneity acts as an implicit diversification mechanism, reducing the likelihood that all particles converge toward the same adaptation basin. As shown in~\cref{tab:noAug}, combining optimizer diversification with explicit regularization (SVGD, Grad, or KL) further improves robustness under batch size 1 adaptation.

\subsection{Input-Level Diversification}

\begin{table*}[!tbp]

    \caption{Comparison with baseline methods on ImageNet-C wild scenarios (severity level 5) under batch size 1, evaluating the impact of augmentation, averaged over five random seeds, on accuracy (\%).}
    \label{tab:augoptim}
    \centering
    \setlength{\tabcolsep}{1.6pt}
    \scalebox{0.48}{
        \begin{tabular}{l|ccc|cccc|cccc|cccc|c}
            \multicolumn{1}{c}{} & \multicolumn{3}{c}{Noise} & \multicolumn{4}{c}{Blur} & \multicolumn{4}{c}{Weather} & \multicolumn{4}{c}{Digital} & \\
            \textbf{Batch Size 1} & \textbf{Gauss.} & \textbf{Shot} & \textbf{Impl.} & \textbf{Defoc.} & \textbf{Glass} & \textbf{Motion} & \textbf{Zoom} & \textbf{Snow} & \textbf{Frost} & \textbf{Fog} & \textbf{Brit.} & \textbf{Contr.} & \textbf{Elastic} & \textbf{Pixel} & \textbf{JPEG} & \textbf{Avg.} \\
            \hline \hline     
            VitBase-LN &9.51 &6.70 &8.21 &28.88 &23.40 &33.91 &27.11 &15.90 &26.48 &47.20 &54.70 &44.11 &30.51 &44.50 &47.81 &29.91 \\
            {Tent}  & 51.15\textsubscript{\scriptsize{$\pm$0.3}} & 51.78\textsubscript{\scriptsize{$\pm$0.2}} & 52.50\textsubscript{\scriptsize{$\pm$0.2}} & 52.18\textsubscript{\scriptsize{$\pm$0.1}} & 47.68\textsubscript{\scriptsize{$\pm$0.2}} & 56.28\textsubscript{\scriptsize{$\pm$0.3}} & 49.06\textsubscript{\scriptsize{$\pm$0.3}} &  8.78\textsubscript{\scriptsize{$\pm$0.4}} & 15.92\textsubscript{\scriptsize{$\pm$0.1}} & 67.25\textsubscript{\scriptsize{$\pm$0.1}} & 73.45\textsubscript{\scriptsize{$\pm$0.2}}& 66.64\textsubscript{\scriptsize{$\pm$0.3}} & 52.09\textsubscript{\scriptsize{$\pm$0.4}} & 64.93\textsubscript{\scriptsize{$\pm$0.2}} & 63.98\textsubscript{\scriptsize{$\pm$0.1}} & 51.58 \\
            {SAR}   & 50.25\textsubscript{\scriptsize{$\pm$0.4}} & 50.65\textsubscript{\scriptsize{$\pm$0.7}} & 51.85\textsubscript{\scriptsize{$\pm$0.3}} & 51.61\textsubscript{\scriptsize{$\pm$0.2}} & 48.92\textsubscript{\scriptsize{$\pm$0.1}} & 56.71\textsubscript{\scriptsize{$\pm$0.1}} & 50.55\textsubscript{\scriptsize{$\pm$0.3}} & 20.45\textsubscript{\scriptsize{$\pm$0.4}} & 54.45\textsubscript{\scriptsize{$\pm$0.6}} & 67.42\textsubscript{\scriptsize{$\pm$0.9}} & 74.92\textsubscript{\scriptsize{$\pm$0.7}} & 65.84\textsubscript{\scriptsize{$\pm$0.0}} & 54.68\textsubscript{\scriptsize{$\pm$0.1}} & 66.57\textsubscript{\scriptsize{$\pm$0.1}} & 64.91\textsubscript{\scriptsize{$\pm$0.1}} & 55.31 \\
            {\deyoabb}  & 52.61\textsubscript{\scriptsize{$\pm$0.7}} & 53.50\textsubscript{\scriptsize{$\pm$1.7}} & 53.46\textsubscript{\scriptsize{$\pm$0.8}} & 54.81\textsubscript{\scriptsize{$\pm$0.1}} & 55.42\textsubscript{\scriptsize{$\pm$0.1}} & 61.92\textsubscript{\scriptsize{$\pm$0.1}} & 38.37\textsubscript{\scriptsize{$\pm$4.8}} & 64.55\textsubscript{\scriptsize{$\pm$0.1}} & 62.93\textsubscript{\scriptsize{$\pm$0.0}} & 70.91\textsubscript{\scriptsize{$\pm$0.1}} & 76.19\textsubscript{\scriptsize{$\pm$0.0}} & 60.14\textsubscript{\scriptsize{$\pm$0.1}} & 65.17\textsubscript{\scriptsize{$\pm$0.1}} & 71.00\textsubscript{\scriptsize{$\pm$0.1}} & 67.56\textsubscript{\scriptsize{$\pm$0.3}} & 60.57 \\
            \hline
            {Naive+Aug}   & 54.21\textsubscript{\scriptsize{$\pm$0.3}} & 55.23\textsubscript{\scriptsize{$\pm$0.4}} & 55.13\textsubscript{\scriptsize{$\pm$0.4}} & 55.38\textsubscript{\scriptsize{$\pm$0.1}}& 55.61\textsubscript{\scriptsize{$\pm$0.1}} & 62.58\textsubscript{\scriptsize{$\pm$0.1}} & 43.11\textsubscript{\scriptsize{$\pm$0.8}} & 65.49\textsubscript{\scriptsize{$\pm$0.1}} & 64.61\textsubscript{\scriptsize{$\pm$0.1}} & 71.67\textsubscript{\scriptsize{$\pm$0.1}} & 76.38\textsubscript{\scriptsize{$\pm$0.1}} & 68.41\textsubscript{\scriptsize{$\pm$0.1}} & 66.39\textsubscript{\scriptsize{$\pm$0.1}} & 71.29\textsubscript{\scriptsize{$\pm$0.1}} & 68.37\textsubscript{\scriptsize{$\pm$0.2}} & 62.26 \\
            {\svgd+Aug}  & 54.72\textsubscript{\scriptsize{$\pm$0.3}} & 56.14\textsubscript{\scriptsize{$\pm$0.3}} & 56.17\textsubscript{\scriptsize{$\pm$0.2}} & 57.10\textsubscript{\scriptsize{$\pm$0.1}} & 57.52\textsubscript{\scriptsize{$\pm$0.1}} & 63.50\textsubscript{\scriptsize{$\pm$0.1}} & 43.13\textsubscript{\scriptsize{$\pm$0.5}} & 66.33\textsubscript{\scriptsize{$\pm$0.1}} & 65.14\textsubscript{\scriptsize{$\pm$0.0}} & 73.36\textsubscript{\scriptsize{$\pm$0.1}} & 77.61\textsubscript{\scriptsize{$\pm$0.1}} & 69.26\textsubscript{\scriptsize{$\pm$0.1}} & 67.61\textsubscript{\scriptsize{$\pm$0.1}} & 72.64\textsubscript{\scriptsize{$\pm$0.0}} & 69.24\textsubscript{\scriptsize{$\pm$0.1}} & \underline{63.30} \\
            {\kl+Aug}    & 54.54\textsubscript{\scriptsize{$\pm$0.3}} & 55.63\textsubscript{\scriptsize{$\pm$0.4}} & 55.98\textsubscript{\scriptsize{$\pm$0.3}} & 56.96\textsubscript{\scriptsize{$\pm$0.1}} & 57.25\textsubscript{\scriptsize{$\pm$0.1}} & 63.50\textsubscript{\scriptsize{$\pm$0.2}} & 41.01\textsubscript{\scriptsize{$\pm$0.5}} & 66.20\textsubscript{\scriptsize{$\pm$0.1}} & 64.96\textsubscript{\scriptsize{$\pm$0.1}} & 73.46\textsubscript{\scriptsize{$\pm$0.1}} & 77.52\textsubscript{\scriptsize{$\pm$0.1}} & 69.25\textsubscript{\scriptsize{$\pm$0.1}} & 68.22\textsubscript{\scriptsize{$\pm$0.1}} & 72.53\textsubscript{\scriptsize{$\pm$0.1}} & 69.53\textsubscript{\scriptsize{$\pm$0.1}} & 63.10 \\
            \rowcolor{LightCyan}
            \grad+Aug   & 56.07\textsubscript{\scriptsize{$\pm$0.2}} & 56.99\textsubscript{\scriptsize{$\pm$0.3}} & 57.32\textsubscript{\scriptsize{$\pm$0.2}} & 57.35\textsubscript{\scriptsize{$\pm$0.0}} & 57.89\textsubscript{\scriptsize{$\pm$0.1}} & 63.85\textsubscript{\scriptsize{$\pm$0.1}} & 60.24\textsubscript{\scriptsize{$\pm$0.2}} & 66.28\textsubscript{\scriptsize{$\pm$0.0}} & 64.97\textsubscript{\scriptsize{$\pm$0.1}} & 73.30\textsubscript{\scriptsize{$\pm$0.1}} & 77.56\textsubscript{\scriptsize{$\pm$0.1}} & 69.29\textsubscript{\scriptsize{$\pm$0.1}} & 67.63\textsubscript{\scriptsize{$\pm$0.0}} & 72.68\textsubscript{\scriptsize{$\pm$0.0}} & 69.39\textsubscript{\scriptsize{$\pm$0.1}} & \textbf{64.72} \\
            \hline
        \end{tabular}
    }

\end{table*}

\begin{table*}[htb]
\caption{Comparison of baselines and WaTT with our diversified variant of WaTT on CIFAR-100-C in terms of accuracy (\%).}
\label{tab:watt}
\centering
\scriptsize
\setlength{\tabcolsep}{2.5pt}
\renewcommand{\arraystretch}{1.2}
\scalebox{0.62}{
\begin{tabular}{l|ccccccccccc}
\toprule
\texttt{CIFAR-100-C} & CLIP & TENT & TPT & TDA & DiffTPT & SAR & CLIPArTT & WATT-P & WATT-S & WATT-S+Aug & WATT-S+Aug+Reg. \\
\midrule
Gaussian Noise & 14.80 & 14.38\textsubscript{$\pm$0.14} & 14.03\textsubscript{$\pm$0.10} & 8.20\textsubscript{$\pm$0.35} & 21.40\textsubscript{$\pm$0.08} & 15.85\textsubscript{$\pm$0.06} & 25.32\textsubscript{$\pm$0.14} & 31.28\textsubscript{$\pm$0.03} & 32.07\textsubscript{$\pm$0.23} & 34.07\textsubscript{$\pm$0.03} & 34.91\textsubscript{$\pm$0.33} \\
Shot noise & 16.03 & 17.34\textsubscript{$\pm$0.27} & 15.25\textsubscript{$\pm$0.17} & 9.58\textsubscript{$\pm$0.43} & 24.17\textsubscript{$\pm$0.49} & 17.41\textsubscript{$\pm$0.05} & 27.90\textsubscript{$\pm$0.05} & 33.44\textsubscript{$\pm$0.11} & 34.36\textsubscript{$\pm$0.11} & 35.87\textsubscript{$\pm$0.13} & 36.79\textsubscript{$\pm$0.26} \\
Impulse Noise & 13.85 & 10.03\textsubscript{$\pm$0.13} & 13.01\textsubscript{$\pm$0.13} & 7.63\textsubscript{$\pm$0.19} & 16.87\textsubscript{$\pm$0.24} & 14.90\textsubscript{$\pm$0.09} & 25.62\textsubscript{$\pm$0.09} & 29.40\textsubscript{$\pm$0.11} & 30.33\textsubscript{$\pm$0.03} & 32.25\textsubscript{$\pm$0.05} & 33.81\textsubscript{$\pm$0.19} \\
Defocus blur & 36.74 & 49.05\textsubscript{$\pm$0.07} & 37.07\textsubscript{$\pm$0.17} & 25.59\textsubscript{$\pm$0.41} & 20.30\textsubscript{$\pm$0.30} & 42.00\textsubscript{$\pm$0.04} & 49.88\textsubscript{$\pm$0.23} & 52.32\textsubscript{$\pm$0.28} & 52.99\textsubscript{$\pm$0.16} & 53.71\textsubscript{$\pm$0.23} & 54.16\textsubscript{$\pm$0.28} \\
Glass blur & 14.19 & 3.71\textsubscript{$\pm$0.07} & 16.41\textsubscript{$\pm$0.02} & 9.83\textsubscript{$\pm$0.05} & 15.57\textsubscript{$\pm$0.46} & 15.07\textsubscript{$\pm$0.02} & 27.89\textsubscript{$\pm$0.03} & 31.20\textsubscript{$\pm$0.12} & 32.15\textsubscript{$\pm$0.30} & 35.12\textsubscript{$\pm$0.30} & 35.61\textsubscript{$\pm$0.33} \\
Motion blur & 36.14 & 46.62\textsubscript{$\pm$0.27} & 37.52\textsubscript{$\pm$0.23} & 28.92\textsubscript{$\pm$0.18} & 21.10\textsubscript{$\pm$0.64} & 39.52\textsubscript{$\pm$0.06} & 47.93\textsubscript{$\pm$0.14} & 49.72\textsubscript{$\pm$0.15} & 50.53\textsubscript{$\pm$0.12} & 51.87\textsubscript{$\pm$0.20} & 51.98\textsubscript{$\pm$0.23} \\
Zoom blur & 40.24 & 51.84\textsubscript{$\pm$0.15} & 42.09\textsubscript{$\pm$0.11} & 31.08\textsubscript{$\pm$0.36} & 25.53\textsubscript{$\pm$0.36} & 45.40\textsubscript{$\pm$0.06} & 52.70\textsubscript{$\pm$0.06} & 54.72\textsubscript{$\pm$0.04} & 55.30\textsubscript{$\pm$0.22} & 56.29\textsubscript{$\pm$0.05} & 56.27\textsubscript{$\pm$0.11} \\
Snow & 38.95 & 46.71\textsubscript{$\pm$0.21} & 43.25\textsubscript{$\pm$0.32} & 32.94\textsubscript{$\pm$0.12} & 28.83\textsubscript{$\pm$0.37} & 43.56\textsubscript{$\pm$0.04} & 49.72\textsubscript{$\pm$0.17} & 51.79\textsubscript{$\pm$0.04} & 52.77\textsubscript{$\pm$0.15} & 54.09\textsubscript{$\pm$0.12} & 55.21\textsubscript{$\pm$0.17} \\
Frost & 40.56 & 44.90\textsubscript{$\pm$0.27} & 43.31\textsubscript{$\pm$0.31} & 34.84\textsubscript{$\pm$0.25} & 31.60\textsubscript{$\pm$0.32} & 41.70\textsubscript{$\pm$0.06} & 49.63\textsubscript{$\pm$0.10} & 53.04\textsubscript{$\pm$0.24} & 53.79\textsubscript{$\pm$0.21} & 54.63\textsubscript{$\pm$0.13} & 54.78\textsubscript{$\pm$0.12} \\
Fog & 38.00 & 47.31\textsubscript{$\pm$0.18} & 30.81\textsubscript{$\pm$0.31} & 31.13\textsubscript{$\pm$0.16} & 16.60\textsubscript{$\pm$0.43} & 40.36\textsubscript{$\pm$0.14} & 48.77\textsubscript{$\pm$0.04} & 50.78\textsubscript{$\pm$0.21} & 51.49\textsubscript{$\pm$0.21} & 53.32\textsubscript{$\pm$0.12} & 53.91\textsubscript{$\pm$0.16} \\
Brightness & 48.18 & 60.58\textsubscript{$\pm$0.18} & 50.23\textsubscript{$\pm$0.11} & 42.36\textsubscript{$\pm$0.36} & 31.20\textsubscript{$\pm$0.16} & 52.77\textsubscript{$\pm$0.14} & 61.27\textsubscript{$\pm$0.08} & 62.65\textsubscript{$\pm$0.25} & 63.57\textsubscript{$\pm$0.21} & 64.34\textsubscript{$\pm$0.34} & 65.13\textsubscript{$\pm$0.37} \\
Contrast & 29.53 & 45.90\textsubscript{$\pm$0.18} & 28.09\textsubscript{$\pm$0.13} & 19.03\textsubscript{$\pm$0.32} & 7.70\textsubscript{$\pm$0.22} & 28.41\textsubscript{$\pm$0.11} & 48.55\textsubscript{$\pm$0.24} & 51.34\textsubscript{$\pm$0.14} & 52.76\textsubscript{$\pm$0.27} & 54.99\textsubscript{$\pm$0.20} & 55.07\textsubscript{$\pm$0.26} \\
Elastic & 26.33 & 33.09\textsubscript{$\pm$0.08} & 28.12\textsubscript{$\pm$0.17} & 18.88\textsubscript{$\pm$0.24} & 21.60\textsubscript{$\pm$0.54} & 26.77\textsubscript{$\pm$0.14} & 37.45\textsubscript{$\pm$0.30} & 39.97\textsubscript{$\pm$0.30} & 40.90\textsubscript{$\pm$0.43} & 43.02\textsubscript{$\pm$0.22} & 43.01\textsubscript{$\pm$0.27} \\
Pixelate & 21.98 & 26.47\textsubscript{$\pm$0.06} & 20.43\textsubscript{$\pm$0.14} & 14.59\textsubscript{$\pm$0.22} & 22.83\textsubscript{$\pm$0.31} & 23.88\textsubscript{$\pm$0.09} & 33.88\textsubscript{$\pm$0.16} & 39.59\textsubscript{$\pm$0.09} & 40.97\textsubscript{$\pm$0.16} & 44.29\textsubscript{$\pm$0.14} & 44.81\textsubscript{$\pm$0.17} \\  
JPEG comp. & 25.91 & 29.89\textsubscript{$\pm$0.07} & 28.82\textsubscript{$\pm$0.09} & 17.56\textsubscript{$\pm$0.11} & 31.77\textsubscript{$\pm$0.45} & 27.20\textsubscript{$\pm$0.04} & 36.07\textsubscript{$\pm$0.32} & 38.99\textsubscript{$\pm$0.16} & 39.59\textsubscript{$\pm$0.08} & 41.67\textsubscript{$\pm$0.19} & 42.48\textsubscript{$\pm$0.24} \\

\midrule
Mean & 29.43 & 35.19 & 30.46 & 22.08 & 22.89 & 31.92 & 41.51 & 44.68 & 45.57 & 47.30 & 47.86 \\
\bottomrule
\end{tabular}
}
\end{table*}

We further introduce diversification at the input level by incorporating structured data augmentations during adaptation. 
Unlike prior work that employs augmentations primarily for filtering or consistency regularization~\cite{deyo2024}, we use them as optimization-time perturbations that alter the local geometry of the entropy landscape experienced by each particle. Given an input $\mathbf{x}$, we construct a set of transformed views
\[
\mathcal{A}(\mathbf{x}) = \{\mathbf{x}, \mathbf{x}^{(h)}, \mathbf{x}^{(v)}\},
\]
where $\mathbf{x}^{(h)}$ and $\mathbf{x}^{(v)}$ denote horizontal and vertical flips. 
For each particle $\theta_i$, adaptation is performed over all views:
\begin{equation}
\ell_{\text{aug}}(X; \theta_i)
=
\sum_{\mathbf{x}' \in \mathcal{A}(\mathbf{x})}
\ell(\mathbf{x}'; \theta_i).
\end{equation}

\noindent The augmented entropy loss replaces $\ell(X;\theta_i)$ in the population objective (Eq.~\ref{eq:pop_objective}), while diversification regularizers remain unchanged.

Input diversification influences adaptation in two complementary ways.
First, it stabilizes entropy minimization by preventing overfitting to a single input configuration.
Second, because augmentations perturb gradients differently for each particle, the repulsive regularizers (like SVGD, Grad, etc.) amplify trajectory separation in parameter space. Consequently, particles explore distinct local minima corresponding to different augmented perspectives of the target stream. As reported in~\cref{tab:augoptim}, combining input diversification with gradient-based repulsion yields the strongest robustness under batch-size-one adaptation. Across corruption types, the proposed configuration consistently outperforms prior TTA methods and improves over DeYO by up to 4\%. Furthermore, compared to WaTT~\cite{osowiechi2024watt}, which adapts only the text encoder of CLIP, our framework operates directly on the image encoder and achieves an additional 2\% average accuracy gain, demonstrating the complementary benefit of particle-level diversification.
It should be noted that our framework maintains $K$ adaptation particles whose predictions are
aggregated at inference time. As a result, the computational cost scales approximately linearly with $K$. In our experiments, we use $K=3$, resulting in roughly a $3\times$ inference overhead compared to single-model TTA methods.
However, since only the normalization layers are updated while the remaining network parameters remain frozen, the additional optimization cost remains lightweight. A more detailed discussion is provided in the Appendix~\ref{append:runtime} and Appendix~\ref{append:limitation}.

\section{Ablation Study}
\label{sec:ablation}








\begin{figure}[t]
\centering

\begin{subfigure}{0.48\linewidth}
    \centering
    \includegraphics[width=\linewidth]{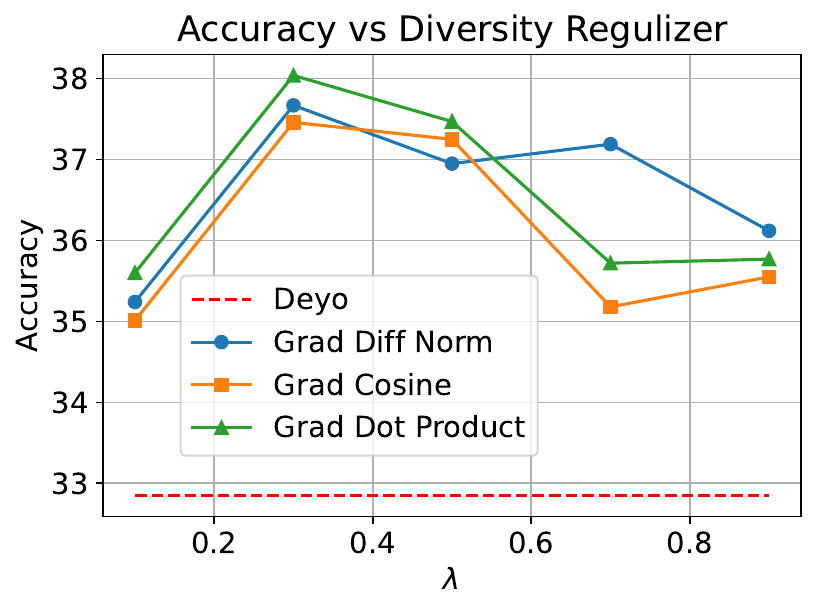}
     \caption{Comparison of gradient-based diversity variants on ImageNet-C with zoom blur corruption (severity 5).}
    \label{fig:lambda}
\end{subfigure}
\hfill
\begin{subfigure}{0.48\linewidth}
    \centering
    \includegraphics[width=\linewidth, height=4.4cm]{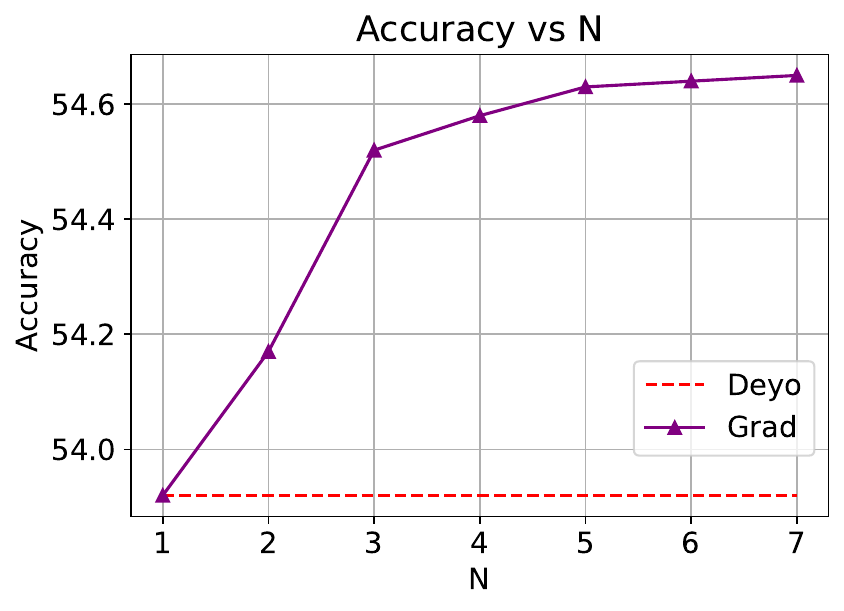}
     \caption{Impact of the particle collection size $N$ on ImageNet-C with Gaussian corruption (severity 5).}
    \label{fig:particles}
\end{subfigure}

\caption{Ablation studies of diversification design choices on ImageNet-C under the wild test-time adaptation scenario (batch size $1$) using a ViTBase-LN backbone.}
\label{fig:diversification_analysis}

\end{figure}

\noindent\textbf{Effect of Gradient Diversity Formulation.}
We evaluate three formulations of gradient-based diversification: 
(i) pairwise dot-product alignment (Eq.~\ref{eq:grad_div}), 
(ii) $\ell_2$ difference between gradients, and 
(iii) cosine similarity.~\Cref{fig:lambda} reports results on ImageNet-C (Zoom Blur, severity level 5) across grid-selected values of the regularization coefficient $\lambda$.
Among the three variants, the dot-product formulation achieves the highest accuracy across most values of $\lambda$. 
Unlike cosine similarity, which normalizes magnitude information, and difference norms, which penalize scale rather than directional alignment, the dot-product directly discourages shared descent directions. 
This suggests that reducing gradient alignment, rather than merely increasing gradient magnitude disparity, is more effective for preventing trajectory collapse under entropy minimization.

Performance peaks around $\lambda=0.3$, indicating a balanced trade-off between entropy minimization and functional repulsion. 
Very large $\lambda$ values degrade accuracy, confirming that excessive repulsion can hinder adaptation.

\noindent\textbf{Sensitivity to Collection Size $N$.}
We analyze the effect of the number of particles on Gaussian noise corruption (severity level 5).
As shown in~\cref{fig:particles}, accuracy improves substantially when increasing from $N=1$ (single-model adaptation) to $N=3$, demonstrating the benefit of multi-trajectory exploration. 
Beyond $N=3$, gains become marginal relative to computational overhead. 
This saturation behavior indicates that a small number of interacting particles is sufficient to capture multiple plausible low-entropy modes. 
Accordingly, we adopt $N=3$ in all experiments to balance robustness and efficiency.

\begin{table}[!tbp]
    \caption{Comparisons with baseline methods on ImageNet-C wild scenario at severity level 5 under batch size 1 across part of corruption types regarding mean of accuracy(\%).}
    \label{tab:tent-grad}
    \centering
    \scalebox{0.64}{
        \begin{tabular}{l|ccc|cccc|c}
            \multicolumn{1}{c}{} & \multicolumn{3}{c}{Noise} & \multicolumn{4}{c}{Blur} \\
            \textbf{Batch Size 1} & \textbf{Gauss.} & \textbf{Shot} & \textbf{Impl.} & \textbf{Defoc.} & \textbf{Glass} & \textbf{Motion} & \textbf{Zoom}& \textbf{Avg.} \\
            \hline \hline
            VitBase-LN & 9.51&6.72 &8.21 &29.01 &23.41 &33.87 & 27.11& 19.69 \\
            {Tent}  & 51.15 & 51.78 & 52.50 & 52.18 & 47.68 & 56.28 & 49.06 & 51.52 \\
            {SAR}   & 50.25 & 50.65 & 51.85 & 51.61 & 48.92 & 56.71 & 50.55 & 51.51 \\
            {\deyoabb}  & 52.61 & 53.50 & 53.46 & 54.81 & 55.42 & 61.92 & 38.37 & \underline{52.87} \\
            {Tent + \grad} & 52.23 & 52.97 & 53.45 & 54.38 & 51.40 & 58.49 & 52.54& \textbf{53.63} \\
            \hline
        \end{tabular}
            }

\end{table}

\noindent\textbf{Plug-in Evaluation on Tent.}
To verify that gradient diversification is not specific to our ensemble initialization, we integrate the proposed gradient repulsion term into Tent~\cite{wang2021tent}.~\Cref{tab:tent-grad} shows that Tent-\texttt{Grad} consistently improves over vanilla Tent and achieves competitive performance with DeYO.
This demonstrates that gradient-based functional diversification acts as a general stabilization mechanism for entropy-based TTA, independent of the underlying baseline.
\section{Conclusion}
\label{sec:concl}

We revisited entropy-based test-time adaptation through the lens of underspecification and argued that conventional single-model adaptation implicitly performs point-estimate (MAP-like) inference under highly underconstrained objectives. Such point estimates are unstable under distribution shift, often leading to trajectory collapse and reliance on spurious modes. To address this limitation, we proposed a population-based diversification framework that maintains multiple interacting adaptation particles. By introducing structured diversification at the parameter, functional (gradient), optimizer, and input levels, our method transforms TTA from brittle single-trajectory optimization into controlled multi-hypothesis exploration. Extensive experiments across ImageNet-C and related benchmarks demonstrate consistent improvements under challenging settings, including batch-size-one adaptation, label distribution shifts, and mixed corruptions. 
Ablation studies further confirm that gradient-based functional repulsion is particularly effective in preventing trajectory alignment under entropy minimization. Overall, our results suggest that treating test-time adaptation as a multi-hypothesis inference problem is essential for reliable deployment under distribution shift. We hope this perspective motivates future work toward principled population-based adaptation strategies that improve robustness under underspecification.
\bibliographystyle{splncs04}
\bibliography{__main}

@String(CVPR  = {IEEE Conf. Comput. Vis. Pattern Recog.})

@String(NeurIPS = {Adv. Neural Inform. Process. Syst.})

@String(ICML  = {Int. Conf. Mach. Learn.})

@String(ICLR  = {Int. Conf. Learn. Represent.})

@String(IJCAI = {IJCAI})

@String(JMLR  = {J. Mach. Learn. Res.})

@String(CVPR  = {CVPR})

@String(NeurIPS = {NeurIPS})

@String(ICML  = {ICML})

@String(ICLR  = {ICLR})

@String(JMLR  = {JMLR})

@String(CVPR= {IEEE Conf. Comput. Vis. Pattern Recog.})

@String(ICLR = {Int. Conf. Learn. Represent.})

@article{nado-2020,
  author       = {Zachary Nado and
                  Shreyas Padhy and
                  D. Sculley and
                  Alexander D'Amour and
                  Balaji Lakshminarayanan and
                  Jasper Snoek},
  title        = {Evaluating Prediction-Time Batch Normalization for Robustness under
                  Covariate Shift},
  journal      = {CoRR},
  volume       = {abs/2006.10963},
  year         = {2020},
  url          = {https://arxiv.org/abs/2006.10963},
  eprinttype    = {arXiv},
  eprint       = {2006.10963},
  bibsource    = {dblp computer science bibliography, https://dblp.org}
}

@inproceedings{wang2021tent,
  title={Tent: Fully Test-Time Adaptation by Entropy Minimization},
  author={Dequan Wang and Evan Shelhamer and Shaoteng Liu and Bruno Olshausen and Trevor Darrell},
  booktitle={International Conference on Learning Representations},
  year={2021},
  url={https://openreview.net/forum?id=uXl3bZLkr3c}
}

@inproceedings{niu2023towards,
  title={Towards Stable Test-Time Adaptation in Dynamic Wild World},
  author={Niu, Shuaicheng and Wu, Jiaxiang and Zhang, Yifan and Wen, Zhiquan and Chen, Yaofo and Zhao, Peilin and Tan, Mingkui},
  booktitle = {Internetional Conference on Learning Representations},
  year = {2023}
}

@inproceedings{deyo2024,
  title={Entropy is not Enough for Test-Time Adaptation: From the Perspective of Disentangled Factors},
  author={Jonghyun Lee and Dahuin Jung and Saehyung Lee and Junsung Park and Juhyeon Shin and Uiwon Hwang and Sungroh Yoon},
  booktitle={The Twelfth International Conference on Learning Representations},
  year={2024},
  url={https://openreview.net/forum?id=9w3iw8wDuE}
}

@inproceedings{pres2024entropy,
  title={The Entropy Enigma: Success and Failure of Entropy Minimization},
  author={Press, Ori and Shwartz-Ziv, Ravid and LeCun, Yann and Bethge, Matthias},
  booktitle={International Conference on Machine Learning},
  year={2024},
  organization={PMLR}
}

@article{jian2023survey,
  author       = {Jian Liang and
                  Ran He and
                  Tieniu Tan},
  title        = {A Comprehensive Survey on Test-Time Adaptation under Distribution
                  Shifts},
  journal      = {CoRR},
  volume       = {abs/2303.15361},
  year         = {2023},
  url          = {https://doi.org/10.48550/arXiv.2303.15361},
  doi          = {10.48550/ARXIV.2303.15361},
  eprinttype    = {arXiv},
  eprint       = {2303.15361},
  bibsource    = {dblp computer science bibliography, https://dblp.org}
}

@article{shannon,
  author = {Shannon, Claude Elwood},
  journal = {The Bell System Technical Journal},
  pages = {379--423},
  title = {A Mathematical Theory of Communication},
  url = {http://plan9.bell-labs.com/cm/ms/what/shannonday/shannon1948.pdf},
  urldate = {2003-04-22},
  volume = 27,
  year = 1948
}

@inproceedings{clip2021,
  author       = {Alec Radford and
                  Jong Wook Kim and
                  Chris Hallacy and
                  Aditya Ramesh and
                  Gabriel Goh and
                  Sandhini Agarwal and
                  Girish Sastry and
                  Amanda Askell and
                  Pamela Mishkin and
                  Jack Clark and
                  Gretchen Krueger and
                  Ilya Sutskever},
  editor       = {Marina Meila and
                  Tong Zhang},
  title        = {Learning Transferable Visual Models From Natural Language Supervision},
  booktitle    = {Proceedings of the 38th International Conference on Machine Learning,
                  {ICML} 2021, 18-24 July 2021, Virtual Event},
  series       = {Proceedings of Machine Learning Research},
  volume       = {139},
  pages        = {8748--8763},
  publisher    = {{PMLR}},
  year         = {2021},
  url          = {http://proceedings.mlr.press/v139/radford21a.html},
  bibsource    = {dblp computer science bibliography, https://dblp.org}
}

@inproceedings{vit2021,
  author       = {Alexey Dosovitskiy and
                  Lucas Beyer and
                  Alexander Kolesnikov and
                  Dirk Weissenborn and
                  Xiaohua Zhai and
                  Thomas Unterthiner and
                  Mostafa Dehghani and
                  Matthias Minderer and
                  Georg Heigold and
                  Sylvain Gelly and
                  Jakob Uszkoreit and
                  Neil Houlsby},
  title        = {An Image is Worth 16x16 Words: Transformers for Image Recognition
                  at Scale},
  booktitle    = {9th International Conference on Learning Representations, {ICLR} 2021,
                  Virtual Event, Austria, May 3-7, 2021},
  publisher    = {OpenReview.net},
  year         = {2021},
  url          = {https://openreview.net/forum?id=YicbFdNTTy},
  bibsource    = {dblp computer science bibliography, https://dblp.org}
}

@inproceedings{imagenetc2019,
  author       = {Dan Hendrycks and
                  Thomas G. Dietterich},
  title        = {Benchmarking Neural Network Robustness to Common Corruptions and Perturbations},
  booktitle    = {7th International Conference on Learning Representations, {ICLR} 2019,
                  New Orleans, LA, USA, May 6-9, 2019},
  publisher    = {OpenReview.net},
  year         = {2019},
  url          = {https://openreview.net/forum?id=HJz6tiCqYm},
  bibsource    = {dblp computer science bibliography, https://dblp.org}
}

@inproceedings{uda2015,
  author       = {Yaroslav Ganin and
                  Victor S. Lempitsky},
  editor       = {Francis R. Bach and
                  David M. Blei},
  title        = {Unsupervised Domain Adaptation by Backpropagation},
  booktitle    = {Proceedings of the 32nd International Conference on Machine Learning,
                  {ICML} 2015, Lille, France, 6-11 July 2015},
  series       = {{JMLR} Workshop and Conference Proceedings},
  volume       = {37},
  pages        = {1180--1189},
  publisher    = {JMLR.org},
  year         = {2015},
  url          = {http://proceedings.mlr.press/v37/ganin15.html},
  bibsource    = {dblp computer science bibliography, https://dblp.org}
}

@inproceedings{uda2018,
  author       = {Kuniaki Saito and
                  Kohei Watanabe and
                  Yoshitaka Ushiku and
                  Tatsuya Harada},
  title        = {Maximum Classifier Discrepancy for Unsupervised Domain Adaptation},
  booktitle    = {2018 {IEEE} Conference on Computer Vision and Pattern Recognition,
                  {CVPR} 2018, Salt Lake City, UT, USA, June 18-22, 2018},
  pages        = {3723--3732},
  publisher    = {Computer Vision Foundation / {IEEE} Computer Society},
  year         = {2018},
  url          = {http://openaccess.thecvf.com/content\_cvpr\_2018/html/Saito\_Maximum\_Classifier\_Discrepancy\_CVPR\_2018\_paper.html},
  doi          = {10.1109/CVPR.2018.00392},
  bibsource    = {dblp computer science bibliography, https://dblp.org}
}

@inproceedings{uda2021,
  author       = {Zhen Qiu and
                  Yifan Zhang and
                  Hongbin Lin and
                  Shuaicheng Niu and
                  Yanxia Liu and
                  Qing Du and
                  Mingkui Tan},
  editor       = {Zhi{-}Hua Zhou},
  title        = {Source-free Domain Adaptation via Avatar Prototype Generation and
                  Adaptation},
  booktitle    = {Proceedings of the Thirtieth International Joint Conference on Artificial
                  Intelligence, {IJCAI} 2021, Virtual Event / Montreal, Canada, 19-27
                  August 2021},
  pages        = {2921--2927},
  publisher    = {ijcai.org},
  year         = {2021},
  url          = {https://doi.org/10.24963/ijcai.2021/402},
  doi          = {10.24963/IJCAI.2021/402},
  bibsource    = {dblp computer science bibliography, https://dblp.org}
}

@article{svgd2016,
  title={Stein variational gradient descent: A general purpose bayesian inference algorithm},
  author={Liu, Qiang and Wang, Dilin},
  journal={Advances in neural information processing systems},
  volume={29},
  year={2016}
}

@inproceedings{evading2022,
  author       = {Damien Teney and
                  Ehsan Abbasnejad and
                  Simon Lucey and
                  Anton van den Hengel},
  title        = {Evading the Simplicity Bias: Training a Diverse Set of Models Discovers
                  Solutions with Superior {OOD} Generalization},
  booktitle    = {{IEEE/CVF} Conference on Computer Vision and Pattern Recognition,
                  {CVPR} 2022, New Orleans, LA, USA, June 18-24, 2022},
  pages        = {16740--16751},
  publisher    = {{IEEE}},
  year         = {2022},
  url          = {https://doi.org/10.1109/CVPR52688.2022.01626},
  doi          = {10.1109/CVPR52688.2022.01626},
  bibsource    = {dblp computer science bibliography, https://dblp.org}
}

@article{domain2016,
  author       = {Yaroslav Ganin and
                  Evgeniya Ustinova and
                  Hana Ajakan and
                  Pascal Germain and
                  Hugo Larochelle and
                  Fran{\c{c}}ois Laviolette and
                  Mario Marchand and
                  Victor S. Lempitsky},
  title        = {Domain-Adversarial Training of Neural Networks},
  journal      = {J. Mach. Learn. Res.},
  volume       = {17},
  pages        = {59:1--59:35},
  year         = {2016},
  url          = {https://jmlr.org/papers/v17/15-239.html},
  bibsource    = {dblp computer science bibliography, https://dblp.org}
}

@inproceedings{domain2013,
  author       = {Krikamol Muandet and
                  David Balduzzi and
                  Bernhard Sch{\"{o}}lkopf},
  title        = {Domain Generalization via Invariant Feature Representation},
  booktitle    = {Proceedings of the 30th International Conference on Machine Learning,
                  {ICML} 2013, Atlanta, GA, USA, 16-21 June 2013},
  series       = {{JMLR} Workshop and Conference Proceedings},
  volume       = {28},
  pages        = {10--18},
  publisher    = {JMLR.org},
  year         = {2013},
  url          = {http://proceedings.mlr.press/v28/muandet13.html},
  bibsource    = {dblp computer science bibliography, https://dblp.org}
}

@inproceedings{pitfalls2020,
  author       = {Harshay Shah and
                  Kaustav Tamuly and
                  Aditi Raghunathan and
                  Prateek Jain and
                  Praneeth Netrapalli},
  editor       = {Hugo Larochelle and
                  Marc'Aurelio Ranzato and
                  Raia Hadsell and
                  Maria{-}Florina Balcan and
                  Hsuan{-}Tien Lin},
  title        = {The Pitfalls of Simplicity Bias in Neural Networks},
  booktitle    = {Advances in Neural Information Processing Systems 33: Annual Conference
                  on Neural Information Processing Systems 2020, NeurIPS 2020, December
                  6-12, 2020, virtual},
  year         = {2020},
  url          = {https://proceedings.neurips.cc/paper/2020/hash/6cfe0e6127fa25df2a0ef2ae1067d915-Abstract.html},
  bibsource    = {dblp computer science bibliography, https://dblp.org}
}

@inproceedings{simplicity2017,
  author       = {Devansh Arpit and
                  Stanislaw Jastrzebski and
                  Nicolas Ballas and
                  David Krueger and
                  Emmanuel Bengio and
                  Maxinder S. Kanwal and
                  Tegan Maharaj and
                  Asja Fischer and
                  Aaron C. Courville and
                  Yoshua Bengio and
                  Simon Lacoste{-}Julien},
  editor       = {Doina Precup and
                  Yee Whye Teh},
  title        = {A Closer Look at Memorization in Deep Networks},
  booktitle    = {Proceedings of the 34th International Conference on Machine Learning,
                  {ICML} 2017, Sydney, NSW, Australia, 6-11 August 2017},
  series       = {Proceedings of Machine Learning Research},
  volume       = {70},
  pages        = {233--242},
  publisher    = {{PMLR}},
  year         = {2017},
  url          = {http://proceedings.mlr.press/v70/arpit17a.html},
  bibsource    = {dblp computer science bibliography, https://dblp.org}
}

@inproceedings{simplicity2019,
  author       = {Vaishnavh Nagarajan and
                  J. Zico Kolter},
  editor       = {Hanna M. Wallach and
                  Hugo Larochelle and
                  Alina Beygelzimer and
                  Florence d'Alch{\'{e}}{-}Buc and
                  Emily B. Fox and
                  Roman Garnett},
  title        = {Uniform convergence may be unable to explain generalization in deep
                  learning},
  booktitle    = {Advances in Neural Information Processing Systems 32: Annual Conference
                  on Neural Information Processing Systems 2019, NeurIPS 2019, December
                  8-14, 2019, Vancouver, BC, Canada},
  pages        = {11611--11622},
  year         = {2019},
  url          = {https://proceedings.neurips.cc/paper/2019/hash/05e97c207235d63ceb1db43c60db7bbb-Abstract.html},
  bibsource    = {dblp computer science bibliography, https://dblp.org}
}

@inproceedings{simplicity2020,
  author       = {Shiori Sagawa and
                  Aditi Raghunathan and
                  Pang Wei Koh and
                  Percy Liang},
  title        = {An Investigation of Why Overparameterization Exacerbates Spurious
                  Correlations},
  booktitle    = {Proceedings of the 37th International Conference on Machine Learning,
                  {ICML} 2020, 13-18 July 2020, Virtual Event},
  series       = {Proceedings of Machine Learning Research},
  volume       = {119},
  pages        = {8346--8356},
  publisher    = {{PMLR}},
  year         = {2020},
  url          = {http://proceedings.mlr.press/v119/sagawa20a.html},
  bibsource    = {dblp computer science bibliography, https://dblp.org}
}

@article{lee2024entropy,
  title={Entropy is not enough for test-time adaptation: From the perspective of disentangled factors},
  author={Lee, Jonghyun and Jung, Dahuin and Lee, Saehyung and Park, Junsung and Shin, Juhyeon and Hwang, Uiwon and Yoon, Sungroh},
  journal={arXiv preprint arXiv:2403.07366},
  year={2024}
}

@article{osowiechi2024watt,
  title={WATT: Weight Average Test-Time Adaptation of CLIP},
  author={Osowiechi, David and Noori, Mehrdad and Hakim, Gustavo Adolfo Vargas and Yazdanpanah, Moslem and Bahri, Ali and Cheraghalikhani, Milad and Dastani, Sahar and Beizaee, Farzad and Ayed, Ismail Ben and Desrosiers, Christian},
  journal={arXiv preprint arXiv:2406.13875},
  year={2024}
}

@inproceedings{chen2022contrastive,
  title={Contrastive test-time adaptation},
  author={Chen, Dian and Wang, Dequan and Darrell, Trevor and Ebrahimi, Sayna},
  booktitle={Proceedings of the IEEE/CVF Conference on Computer Vision and Pattern Recognition},
  pages={295--305},
  year={2022}
}

@article{hendrycks2019benchmarking,
  title={Benchmarking neural network robustness to common corruptions and perturbations},
  author={Hendrycks, Dan and Dietterich, Thomas},
  journal={arXiv preprint arXiv:1903.12261},
  year={2019}
}

@inproceedings{recht2019imagenet,
  title={Do imagenet classifiers generalize to imagenet?},
  author={Recht, Benjamin and Roelofs, Rebecca and Schmidt, Ludwig and Shankar, Vaishaal},
  booktitle={International conference on machine learning},
  pages={5389--5400},
  year={2019},
  organization={PMLR}
}

@book{quinonero2009dataset,
  title     = {Dataset Shift in Machine Learning},
  author    = {Quiñonero, Joaquin and Sugiyama, Masashi and Schwaighofer, Anton and Lawrence, Neil D.},
  year      = {2009},
  publisher = {MIT Press},
  address   = {Cambridge, MA},
  isbn      = {9780262170055}
}

@inproceedings{sun2020test,
  title={Test-time training with self-supervision for generalization under distribution shifts},
  author={Sun, Yu and Wang, Xiaolong and Liu, Zhuang and Miller, John and Efros, Alexei and Hardt, Moritz},
  booktitle={International conference on machine learning},
  pages={9229--9248},
  year={2020},
  organization={PMLR}
}

@article{zhang2022memo,
  title={Memo: Test time robustness via adaptation and augmentation},
  author={Zhang, Marvin and Levine, Sergey and Finn, Chelsea},
  journal={Advances in neural information processing systems},
  volume={35},
  pages={38629--38642},
  year={2022}
}

@inproceedings{liu2016svgd,
  title={Stein Variational Gradient Descent: A General Purpose Bayesian Inference Algorithm},
  author={Liu, Qiang and Wang, Dilin},
  booktitle={Advances in Neural Information Processing Systems},
  volume={29},
  year={2016}
}

@misc{ilharco_gabriel_2021_5143773,
  author       = {Ilharco, Gabriel and
                  Wortsman, Mitchell and
                  Wightman, Ross and
                  Gordon, Cade and
                  Carlini, Nicholas and
                  Taori, Rohan and
                  Dave, Achal and
                  Shankar, Vaishaal and
                  Namkoong, Hongseok and
                  Miller, John and
                  Hajishirzi, Hannaneh and
                  Farhadi, Ali and
                  Schmidt, Ludwig},
  title        = {OpenCLIP},
  month        = jul,
  year         = 2021,
  note         = {If you use this software, please cite it as below.},
  publisher    = {Zenodo},
  version      = {0.1},
  doi          = {10.5281/zenodo.5143773},
  url          = {https://doi.org/10.5281/zenodo.5143773}
}

@inproceedings{radford2021learning,
  title     = {Learning Transferable Visual Models From Natural Language Supervision},
  author    = {Alec Radford and Jong Wook Kim and Chris Hallacy and Aditya Ramesh and Gabriel Goh and Sandhini Agarwal and Girish Sastry and Amanda Askell and Pamela Mishkin and Jack Clark and Gretchen Krueger and Ilya Sutskever},
  booktitle = {Proceedings of the 38th International Conference on Machine Learning (ICML)},
  year      = {2021}
}

@inproceedings{hendrycks2021many,
  title={The many faces of robustness: A critical analysis of out-of-distribution generalization},
  author={Hendrycks, Dan and Basart, Steven and Mu, Norman and Kadavath, Saurav and Wang, Frank and Dorundo, Evan and Desai, Rahul and Zhu, Tyler and Parajuli, Samyak and Guo, Mike and others},
  booktitle={Proceedings of the IEEE/CVF international conference on computer vision},
  pages={8340--8349},
  year={2021}
}

@inproceedings{bashkirova2022visda,
  title={Visda-2021 competition: Universal domain adaptation to improve performance on out-of-distribution data},
  author={Bashkirova, Dina and Hendrycks, Dan and Kim, Donghyun and Liao, Haojin and Mishra, Samarth and Rajagopalan, Chandramouli and Saenko, Kate and Saito, Kuniaki and Tayyab, Burhan Ul and Teterwak, Piotr and others},
  booktitle={NeurIPS 2021 Competitions and Demonstrations Track},
  pages={66--79},
  year={2022},
  organization={PMLR}
}

\newpage
\appendix
\renewcommand{\theHsection}{appendix.\Alph{section}}

\clearpage
\centerline{{\Large \textbf{Multi-Hypothesis Test-Time Adaptation to}}}
\centerline{{\Large \textbf{Mitigate Underspecification}}}

\section*{Overview of Materials in the Appendices}


\noindent{A brief overview of additional experimental results and findings is provided in the following appendices.} 

\begin{enumerate}
  \item Related work (Appendix~\ref{append:related})
  \item Underspecification in TTA (Appendix~\ref{append:entropy_underspec_geometry})
  \item Beyond Corruption-Based Shifts (Appendix~\ref{append:beyond_corruptions})
  \item Sample Selection and Hyperparameter Configuration (Appendix~\ref{append:hyperparameter})
  \item Optimization Objective of SVGD (Appendix~\ref{append:svgd})
  \item Additional experiments regarding mild scenarios (Appendix~\ref{append:mild})
  \item Additional experiments associated with wild scenarios for severity level of 3 (Appendix~\ref{append:sev3})
  \item Runtime comparison of methods (Appendix~\ref{append:runtime})
  \item Limitation (Appendix~\ref{append:limitation})
\end{enumerate}

    



   


    


\section{Related Work}\label{append:related}

\subsection{Out-of-distribution Generalization}
Out-of-distribution (OOD) generalization has become a critical research area, as models deployed in real-world scenarios often encounter data that are different from the training distribution. 
Existing research in OOD generalization primarily focuses on domain adaptation and invariant representation learning.

\noindent\textbf{Domain Adaptation \& Generalization}
Assuming the accessibility of the distribution of target data, domain adaptation is a related field to OOD generalization.
It can be considered as a specific case of OOD generalization where some prior knowledge of the target distribution is available.
Techniques like adversarial domain adaptation~\cite{domain2016} attempt to align the feature distributions of the source and target domains using adversarial objectives.
Domain generalization aims to learn models that perform well across a broader set of domains, including those unseen during training. Techniques in this area, such as domain-invariant feature extraction~\cite{domain2013}, optimize for representations that minimize domain-specific information, improving generalization under OOD scenarios.

\noindent\textbf{Unsupervised Domain Adaptation}
 aims to learn generalizable models using
unlabeled data, while simultaneously examining the impact of pre-training on OOD generalization.
Recent approaches have demonstrated promising results by leveraging large-scale unlabeled data to learn robust representation spaces~\cite{uda2015,uda2018,uda2021}. 
However, prior methods struggle to directly address the OOD challenge, as the learned representation spaces often retain domain-specific features used to distinguish negative samples.
These domain-specific features can be ineffective or even harmful to downstream tasks.

\noindent\textbf{\TTA} (\ttaab)~\cite{wang2021tent,jian2023survey,deyo2024} is an emerging research field that involves adapting a pre-trained model from the source domain to unlabeled data in the target domain during testing, which can improve the generalization of machine learning models to new or unseen data distributions without requiring access to labeled data.
Test-time batch normalization (BN) \cite{nado-2020} is one of the earliest methods in \ttaab, where the statistics of the batch normalization layers are updated using the test data, thereby allowing the model to adapt to the test distribution.
TENT \cite{wang2021tent} is an approach that involves entropy minimization, where the model is adapted by minimizing the prediction uncertainty on the test data. 
This method encourages the model to produce more confident predictions on the test set, thereby improving its robustness to distribution shifts.
\cite{niu2023towards} study the failure cases of \ttaab methods and propose a sharpness-aware and reliable entropy minimization method SAR, which removes partial noisy samples with large gradients and encourages model weights to go to a flat minimum so that the model is robust to the remaining noisy samples.
Due to the limitation of entropy~\cite{pres2024entropy}, \cite{deyo2024} first illustrate the limitations of relying solely on entropy as a confidence metric for \ttaab. 
Based on the observation, they further introduce a new \ttaab method called DeYO, which leverages a novel proposed confidence metric, PLPD.

\subsection{Simplicity Bias}
Simplicity bias refers to DNNs' tendency to favor simple patterns over more complex ones in training data.
This occurs because DNNs, especially those trained with gradient-based methods, often prioritize learning features or patterns that are easier to optimize, even if these features are not the most robust or generalizable for the task. 
As a result, simplicity bias limits DNNs' ability to generalize to OOD data, as the simple patterns it learned may not be valid in new, unseen environments.

Research on simplicity bias has explored how it affects model generalization and performance on OOD data. 
\cite{simplicity2017} were among the first to highlight that DNNs tend to prioritize learning simple patterns first, particularly in early training phases. 
This tendency is due to the efficiency of gradient descent, which favors features that are easier to optimize. 
\cite{simplicity2019} demonstrated that simplicity bias could lead models to focus on spurious correlations in training data, which impairs performance on OOD data where these correlations may not hold.

More recently, researchers have investigated approaches to mitigate simplicity bias to improve OOD generalization. 
\cite{simplicity2020} introduced Group DRO, a robust optimization method that forces models to learn features that generalize across different data subgroups, thus reducing reliance on simple, non-robust patterns. 
\cite{pitfalls2020} proposed techniques based on adversarial training, where models are trained against perturbed data points, encouraging them to move beyond simpler features and to learn more generalizable representations.
\cite{evading2022} trained a collection of models and identified only one for inference, which discovered predictive patterns normally missed by a learning algorithm because of the simplicity bias.

\section{Underspecification of Entropy Minimization: A Geometric Characterization}
\label{append:entropy_underspec_geometry}

Entropy minimization is a convenient unsupervised objective for test-time adaptation, but it does not uniquely determine a robust decision function on the target distribution. In particular, the objective encourages \emph{confidence} without enforcing \emph{semantic correctness}. As a result, many parameter settings can achieve essentially identical (near-minimal) entropy while inducing qualitatively different decision rules and target risks. Consistent with the discussion in Sec.~\ref{sec:method}, we formalize this underspecification phenomenon by characterizing the relevant symmetries and the (potentially large) set of entropy-minimizing solutions.

\paragraph{Binary case.}
For clarity, we first consider binary classification (the multi-class extension is discussed at the end of this subsection). Let the model output a logit $z_\theta(x)\in\mathbb{R}$ and predictive probability $p_\theta(x)=\sigma(z_\theta(x))$, where $\sigma(\cdot)$ denotes the sigmoid function. The predictive entropy for an input $x$ is
\begin{equation}
H_\theta(x)
=
- p_\theta(x)\log p_\theta(x)
- \bigl(1-p_\theta(x)\bigr)\log\bigl(1-p_\theta(x)\bigr),
\end{equation}
and the standard entropy-minimization objective over target samples $X^t=\{x_i\}_{i=1}^n$ is
\begin{equation}
\mathcal{L}(\theta)
=
\frac{1}{n}\sum_{i=1}^n H_\theta(x_i).
\end{equation}

\begin{lemma}[Entropy sign symmetry]
\label{lem:entropy_sign_symmetry}
For all $z\in\mathbb{R}$, $H(\sigma(z))=H(\sigma(-z))$.
\end{lemma}

\noindent\textit{Proof.}
Since $\sigma(-z)=1-\sigma(z)$ and binary entropy satisfies $H(p)=H(1-p)$ for all $p\in(0,1)$, the claim follows immediately.
\hfill$\square$

\vspace{.4cm}

Lemma~\ref{lem:entropy_sign_symmetry} implies that entropy depends only on the \emph{magnitude} $|z|$ (i.e., confidence) and is invariant to the \emph{sign} of the logit (i.e., the predicted label). Consequently, entropy minimization alone cannot distinguish between two hypotheses that are equally confident but assign opposite labels on a subset of the target samples. We make this non-identifiability explicit below.

\begin{proposition}[Non-identifiability of entropy-minimizing solutions]
\label{prop:entropy_non_identifiability}
Assume the model is sufficiently expressive such that there exist two parameter settings $\theta_1$ and $\theta_2$ and a constant $M>0$ satisfying $|z_{\theta_1}(x_i)|\ge M$ and $|z_{\theta_2}(x_i)|\ge M$ for all $x_i\in X^t$. Suppose further that there exists a non-empty subset $S\subset\{1,\dots,n\}$ such that
\begin{equation}
\operatorname{sign}\!\bigl(z_{\theta_1}(x_i)\bigr)
=
-\operatorname{sign}\!\bigl(z_{\theta_2}(x_i)\bigr),
\qquad \forall i\in S.
\end{equation}
Then: (i) $\mathcal{L}(\theta_1)\approx \mathcal{L}(\theta_2)\approx 0$ for large $M$; (ii) the induced classifiers disagree on $S$; and (iii) unless the true labels are constant on $S$, the corresponding target risks differ, i.e., $R_t(\theta_1)\neq R_t(\theta_2)$.
\end{proposition}

\noindent\textit{Proof.}
For each $k\in\{1,2\}$ and each $x_i\in X^t$, the condition $|z_{\theta_k}(x_i)|\ge M$ with large $M$ implies $\sigma(z_{\theta_k}(x_i))\to 0$ or $1$, hence $H_{\theta_k}(x_i)\to 0$. Averaging over $i$ yields $\mathcal{L}(\theta_k)\to 0$, so both $\theta_1$ and $\theta_2$ achieve (near-)minimal entropy. By Lemma~\ref{lem:entropy_sign_symmetry}, flipping the sign of a logit leaves the entropy unchanged, so the objective cannot prefer $\theta_1$ over $\theta_2$ based on $\mathcal{L}(\cdot)$ alone. However, the decision rules $\operatorname{sign}(z_{\theta_1}(\cdot))$ and $\operatorname{sign}(z_{\theta_2}(\cdot))$ disagree on $S$ by assumption, and if the ground-truth labels are not constant on $S$, at least one of the two hypotheses must incur strictly larger classification error on that subset, implying $R_t(\theta_1)\neq R_t(\theta_2)$.
\hfill$\square$

\paragraph{Geometric interpretation and implication for test-time adaptation.}
Proposition~\ref{prop:entropy_non_identifiability} formalizes a core limitation of entropy minimization: the objective only enforces high confidence on the target samples and is invariant to certain label-flipping symmetries. Therefore, multiple high-confidence solutions can attain essentially identical entropy values while inducing different decision boundaries and target risks. In overparameterized models, such ambiguity is particularly salient because parameter space often contains directions (or transformations of the classifier head) that can invert the decision rule without materially changing confidence on the observed target samples. As a result, test-time adaptation based solely on entropy minimization should be viewed as an inherently underspecified problem, motivating structured multi-hypothesis exploration or additional constraints that break these symmetries.

\paragraph{Multi-class note.}
For $K>2$ classes, exact global \emph{sign} symmetry does not carry over, but a closely related underspecification persists: the entropy depends only on the predicted probability vector and is invariant to permutations of class identities. In particular, for any permutation $\pi$ of $\{1,\dots,K\}$ and any probability vector $p\in\Delta^{K-1}$, we have $H(p)=H(\pi(p))$. Consequently, high-confidence assignments that map target samples to different class indices can achieve the same (near-zero) entropy while corresponding to different induced labelings and risks. This observation supports the same conclusion: entropy minimization encourages confidence but does not, by itself, identify the semantically correct classifier on the target domain. This theoretical perspective directly motivates the structured multi-hypothesis diversification strategy used in our main method.


\begin{table}[t]
\caption{Extended evaluation on non-corruption distribution shifts. Results are average target-domain accuracy (\%).}
\label{tab:extended_shifts}
\centering
\scalebox{0.9}{
\begin{tabular}{lccccccc}
\toprule
Dataset & Tent & SAR & DeYo & Naive & SVGD & KL & Grad \\
\midrule
ColoredMNIST & 56.30 & 57.28 & 76.88 & 79.21 & 80.36 & 80.21 & \textbf{81.33} \\
Waterbirds   & 81.88 & 83.21 & 86.47 & 87.87 & 88.31 & 88.49 & \textbf{89.21} \\
ImageNet-R   & 44.28 & 43.14 & 46.7 & 48.82 & 49.31 & 49.13 & \textbf{51.34} \\
Visda-2021   & 43.88 & 43.94 & 45.41 & 45.81 & 47.11 & 46.91& \textbf{48.46} \\
\bottomrule
\end{tabular}}
\end{table}

\section{Beyond Corruption-Based Shifts}
\label{append:beyond_corruptions}

To further assess robustness beyond corruption-based benchmarks, we conduct additional experiments on datasets exhibiting different types of distribution shift, including spurious correlations and domain shifts. Specifically, we evaluate on the ColoredMNIST and Waterbirds datasets, which introduce label–attribute correlations through color and background biases, as well as ImageNet-R~\cite{hendrycks2021many} and VisDA-2021~\cite{bashkirova2022visda}, which contain substantial style and domain variations relative to the source distribution. For ColoredMNIST, we use a ResNet-18-BN model with batch size 64, while for Waterbirds we employ a ResNet-50-BN backbone under the same batch setting. For ImageNet-R and VisDA-2021, we follow the wild scenario and use ResNet-50-GN with batch size equal to 1. The adaptation protocol and optimization hyperparameters remain consistent across all methods for fair comparison. Table~\ref{tab:extended_shifts} reports the average target-domain accuracy over five random seeds. Among the diversification strategies, gradient-based diversification (Grad) consistently achieves the best performance, indicating that maintaining diverse adaptation trajectories improves robustness across a broader range of distribution shifts.


\section{Sample Selection and Hyperparameter Configuration}
\label{append:hyperparameter}

To ensure a fair comparison with prior baselines, we adopt the same sample-selection strategy used in previous test-time adaptation methods. In particular, we follow the confidence-based filtering mechanism employed in SAR~\cite{niu2023towards} and DeYo~\cite{deyo2024}, where only \emph{reliable (non-harmful) samples} are used to update the model during adaptation. Specifically, a test sample $\rvx$ is selected for adaptation if it satisfies two conditions: (1) the predictive entropy is below a predefined threshold $\tau_H$, and (2) the pseudo-label probability difference (PLPD) between the original input and its patch-shuffled counterpart exceeds a predefined threshold. Following DeYo~\cite{deyo2024}, the PLPD score is defined as
\begin{equation}
\text{PLPD}_\vtheta(\rvx, \rvx') = (\rvp_\vtheta(\rvx) - \rvp_\vtheta(\rvx'))_{\hat{y}},
\end{equation}
where $\rvx$ denotes the input image, $\rvx'$ is the patch-shuffled version of the same image, $\rvp_\vtheta(\cdot)$ denotes the predicted class probability vector, and $\hat{y}$ is the predicted pseudo-label of the model.

Formally, we update the model using samples satisfying
\[
H(\rvp_\vtheta(\rvx)) < \tau_H
\quad \text{and} \quad
\text{PLPD}_\vtheta(\rvx,\rvx') > \tau_{\text{PLPD}},
\]
where $H(\cdot)$ denotes predictive entropy. Following prior work, we set $\tau_H = 0.4$ and $\tau_{\text{PLPD}} = 0.3$. Using identical filtering thresholds ensures a controlled comparison under consistent sample-selection criteria during adaptation. For our diversification strategies, hyperparameters were selected via grid search. In particular, we set the KL regularization weight to $\lambda_{\text{KL}} = 0.01$, the gradient-based diversification weight to $\lambda_{\text{Grad}} = 0.3$, and the SVGD kernel bandwidth to $0.01$. These values were chosen based on stability and performance considerations and were kept fixed across all experiments once selected.


\section{SVGD and its Update Rules}
\label{append:svgd}

\begin{table*}[!th]
    \caption{Comparisons with baseline methods on CIFAR-10-C and CIFAR-100-C mild senarios at severity level 5 under batch size 64 regarding accuracy (\%).}
    \label{tab:cifar-c}
    \centering
    \scalebox{0.6}{
        \begin{tabular}{l|cccc|c|cc|ccccc|c}
        \multicolumn{1}{c}{} & \multicolumn{4}{c}{Noise} & \multicolumn{1}{c}{Blur} & \multicolumn{2}{c}{Weather} & \multicolumn{4}{c}{Digital} & \multicolumn{1}{c}{} \\
        \textbf{CIFAR-10-C} & \textbf{Gauss.} & \textbf{Shot} & \textbf{Impl.} & \textbf{Speckle} & \textbf{Gauss.} & \textbf{Brit.} & \textbf{Spatter} & \textbf{Contr.} & \textbf{Elastic} & \textbf{Pixel} & \textbf{JPEG} & \textbf{Satruate} & \textbf{Avg.} \\
        \hline \hline
        CLIP-ViT-B-32 &25.21 & 36.70& 31.3& 34.21& 51.02&88.40 &67.18 &76.98 & 31.26& 56.30& 28.90& 43.21&47.55\\
        Tent & 13.54 & 16.01 & 12.67 & 16.04 & 71.45 & 94.79 & 80.09 & 91.95 & 26.60 & 80.23 & 18.47 & 75.85 & 49.81 \\ 
        SAR  & 41.40 & 46.64 & 43.57 & 50.89 & 71.59 & 94.10 & 79.22 & 78.75 & 68.91 & 78.79 & 52.78 & 85.62 & 66.02 \\ 
        \deyoabb & 33.35 & 42.20 & 20.69 & 43.87 & 67.70 & 88.14 & 76.78 & 86.45 & 55.13 & 75.73 & 40.50 & 84.30 & 59.57 \\ 
        \hline
        \ens  & 62.02 & 72.30 & 64.48 & 73.14 & 86.08 & 92.28 & 87.56 & 92.01 & 75.47 & 86.96 & 64.62 & 90.24 & \underline{78.93} \\ 
        \svgd & 61.28 & 72.21 & 61.76 & 72.01 & 79.63 & 92.13 & 86.01 & 89.38 & 78.73 & 85.34 & 64.78 & 90.17 & 77.79 \\ 
        \kl   & 34.63 & 46.21 & 35.82 & 42.85 & 71.58 & 87.52 & 68.90 & 88.43 & 61.27 & 77.86 & 57.40 & 81.46 & 62.83 \\ 
        \rowcolor{LightCyan}
        \grad  & 62.82 & 72.47 & 60.04 & 75.60 & 85.73 & 92.49 & 87.49 & 91.82 & 78.44 & 86.74 & 65.68 & 90.07 & \textbf{79.12} \\ 
        \hline
        \end{tabular}
    }
   
    \bigskip

    \scalebox{0.6}{
        \begin{tabular}{l|cccc|c|cc|ccccc|c}
        \multicolumn{1}{c}{} & \multicolumn{4}{c}{Noise} & \multicolumn{1}{c}{Blur} & \multicolumn{2}{c}{Weather} & \multicolumn{4}{c}{Digital} & \multicolumn{1}{c}{} \\
        \textbf{CIFAR-100-C} & \textbf{Gauss.} & \textbf{Shot} & \textbf{Impl.} & \textbf{Speckle} & \textbf{Gauss.} & \textbf{Brit.} & \textbf{Spatter} & \textbf{Contr.} & \textbf{Elastic} & \textbf{Pixel} & \textbf{JPEG} & \textbf{Satruate} & \textbf{Avg.} \\
        \hline \hline
        CLIP-ViT-B-32 & 5.8& 20.01& 4.24& 11.36& 23.31& 56.62& 21.67&22.23 & 11.39& 27.86& 6.71&19.81&19.25 \\
        Tent &  1.93 &  2.16 &  1.98 &  3.15 & 55.13 & 79.94 & 28.52 &  9.61 &  9.84 & 31.80 &  4.68 & 22.62 &  20.95 \\ 
        SAR  & 11.71 & 16.13 & 13.75 & 18.41 & 59.79 & 80.28 & 55.27 & 27.94 & 42.54 & 52.75 & 23.72 & 33.94 &  36.35 \\ 
        \deyoabb & 10.60 & 25.15 &  4.27 & 21.23 & 69.37 & 78.77 & 56.89 & 70.33 & 19.40 & 60.69 &  9.17 & 63.39 &  40.77 \\ 
        \hline
        \ens  & 11.28 & 31.87 &  8.26 & 31.43 & 73.56 & 79.71 & 67.50 & 74.34 & 40.75 & 68.67 &  8.78 & 64.98 &  \underline{46.76} \\ 
        \svgd &  8.28 & 32.08 &  9.01 & 31.29 & 73.01 & 80.41 & 65.91 & 73.51 & 40.98 & 66.31 &  8.28 & 64.86 &  46.16 \\ 
        \kl   &  9.31 & 31.41 &  7.67 & 30.49 & 72.47 & 77.91 & 66.71 & 73.13 & 40.89 & 67.73 &  8.28 & 64.01 &  45.83 \\ 
        \rowcolor{LightCyan}
        \grad  & 10.88 & 34.27 &  9.44 & 32.58 & 73.56 & 80.21 & 68.15 & 74.68 & 41.58 & 68.91 & 11.31 & 65.04 &  \textbf{47.55} \\ 
        \hline
        \end{tabular}
    }
\end{table*}    
\begin{table*}[!t]
        \caption{Comparisons with baseline methods on ImageNet-C mild scenarios with batch size 64 at severity level 5 regarding accuracy (\%).}
    \label{tab:imagenet-c-mild}
    \centering
    \scalebox{0.6}{
        \begin{tabular}{l|ccc|cccc|cccc|cccc|c}
            \multicolumn{1}{c}{} & \multicolumn{3}{c}{Noise} & \multicolumn{4}{c}{Blur} & \multicolumn{4}{c}{Weather} & \multicolumn{4}{c}{Digital} & \\
            \textbf{Mild} & \textbf{Gauss.} & \textbf{Shot} & \textbf{Impl.} & \textbf{Defoc.} & \textbf{Glass} & \textbf{Motion} & \textbf{Zoom} & \textbf{Snow} & \textbf{Frost} & \textbf{Fog} & \textbf{Brit.} & \textbf{Contr.} & \textbf{Elastic} & \textbf{Pixel} & \textbf{JPEG} & \textbf{Avg.} \\
            \hline \hline
            ResNet-50-BN &27.61 & 25.02& 25.22& 37.89& 16.88& 37.71&35.21 &35.21 &32.07 & 46.68& 69.61& 46.01& 55.58&46.19 & 59.32& 39.76\\
            {Tent} & 29.18 & 31.23 & 30.12 & 28.20 & 27.63 & 41.43 & 49.41 & 47.21 & 41.51 & 57.68 & 67.50 & 29.38 & 54.78 & 58.60 & 52.45 & 43.09 \\
            {SAR} & 28.67 & 31.67 & 30.40 & 28.07 & 26.84 & 42.16 & 49.64 & 47.55 & 42.68 & 57.66 & 67.42 & 38.71 & 54.36 & 58.56 & 52.53 & 43.79 \\
            {\deyoabb} & 34.58 & 36.92 & 36.10 & 32.76 & 33.14 & 47.55 & 51.77 & 51.67 & 45.39 & 59.57 & 66.95 & 45.07 & 57.42 & 60.49 & 54.71 & 47.61 \\
            \hline
            {\ens} & 35.52 & 37.34 & 37.78 & 35.02 & 33.84 & 48.11 & 52.23 & 52.28 & 45.85 & 60.35 & 67.86 & 45.53 & 57.64 & 61.29 & 54.73 & 48.36 \\
            {\svgd} & 36.07 & 38.10 & 37.55 & 35.02 & 34.70 & 48.58 & 52.90 & 52.48 & 46.51 & 60.57 & 67.96 & 46.44 & 58.52 & 61.43 & 55.61 & \underline{48.83} \\
            {\kl} & 36.15 & 38.07 & 37.46 & 34.88 & 34.28 & 48.57 & 52.91 & 52.59 & 46.40 & 60.48 & 67.97 & 45.34 & 58.52 & 61.37 & 55.73 & 48.71 \\
            \rowcolor{LightCyan}
            \grad  & 36.12 & 38.28 & 37.50 & 35.19 & 34.66 & 48.53 & 52.86 & 52.56 & 46.41 & 60.45 & 68.26 & 46.69 & 58.46 & 61.39 & 55.71 & \textbf{48.87} \\
            \hline
        \end{tabular}
    }
\end{table*}

\noindent\textbf{SVGD} or Stein Variational Gradient Descent~\cite{svgd2016}, is a particle-based variational inference algorithm. 
It aims to approximate a target distribution by evolving a set of particles, which are essentially points in the parameter space. 
These particles are iteratively moved to better represent the target distribution using a combination of variational principles and gradient-based optimization.
Mathematically, the update rule for each particle is:
\begin{equation}
    \vtheta^{b}_i = \vtheta^{b-1}_i + \epsilon \hat{\phi}^*(\vtheta^{b-1}_i),
\end{equation}
where 
\begin{equation}
    \begin{split}
        \hat{\phi}^*(\vtheta^{b-1}_i) &= \frac{1}{N} \sum^{N}_{l=1} (k(\vtheta^{b-1}_l, \vtheta^{b-1}_i) \nabla_{\vtheta^{b-1}_l} \log p(\vtheta^{b-1}_l) \\ & + \nabla_{\vtheta^{b-1}_l} k(\vtheta^{b-1}_l, \vtheta^{b-1}_i))
    \end{split}
\end{equation}
Here, $k(\vtheta^{b-1}_l, \vtheta^{b-1}_i)$ is a kernel function that defines the similarity between $\vtheta^{b-1}_l$ and $\vtheta^{b-1}_i$, ensuring that the models don't collapse
and spread out properly.
The first term $k(\vtheta^{b-1}_l, \vtheta^{b-1}_i) \nabla_{\vtheta^{b-1}_l} \log p(\vtheta^{b-1}_l)$ moves each particle in the direction of higher probability in the target distribution by following the gradient of the log-posterior. 
The second term $\nabla_{\vtheta^{b-1}_l} k(\vtheta^{b-1}_l, \vtheta^{b-1}_i)$, at the same time, encourages diversity among the particles by pushing them away from each other.

\section{Mild Scenario}
\label{append:mild}

\textbf{Comparison on Mild Scenario.}
For the mild scenario, the comparison results on CIFAR-10-C, CIFAR-100-C and ImageNet-C are reported in ~\cref{tab:cifar-c,tab:imagenet-c-mild} separately. 
Our \grad setting consistently outperforms the baseline methods across all 12 corruption types on CIFAR-10-C and CIFAR-100-C in terms of accuracy, affirming the effectiveness of our approach. 
Notably, the \ens setting increases model performance by large margins (+12.91\% on CIFAR-10-C and +5.99\% on CIFAR-100-C) without using diversity measures. Here, the backbone is OpenCLIP~\cite{ilharco_gabriel_2021_5143773}, which differs from the version of CLIP proposed by \cite{radford2021learning} that WATT~\cite{osowiechi2024watt} utilized.
On ImageNet-C benchmark, our \grad method exhibits a 1.26\% higher performance on average, even when compared to \deyoabb, which demonstrates the state-of-the-art performance in the mild scenario. It is noted that the mean of accuracies, computed over five different random seeds, has been reported in \cref{tab:cifar-c,tab:imagenet-c-mild}.

\section{Wild Scenarios with Severity Level of 3}
\label{append:sev3}

\begin{table*}[!ht]
    \caption{Comparisons with baseline methods on ImageNet-C wild senarios at severity level 3 under batch size 1 and label shifts regarding accuracy (\%).}
    \label{tab:imagenet-sev3}
    \centering
    \scalebox{0.6}{
        \begin{tabular}{l|ccc|cccc|cccc|cccc|c}
            \multicolumn{1}{c}{} & \multicolumn{3}{c}{Noise} & \multicolumn{4}{c}{Blur} & \multicolumn{4}{c}{Weather} & \multicolumn{4}{c}{Digital} & \\
            \textbf{Batch Size 1} & \textbf{Gauss.} & \textbf{Shot} & \textbf{Impl.} & \textbf{Defoc.} & \textbf{Glass} & \textbf{Motion} & \textbf{Zoom} & \textbf{Snow} & \textbf{Frost} & \textbf{Fog} & \textbf{Brit.} & \textbf{Contr.} & \textbf{Elastic} & \textbf{Pixel} & \textbf{JPEG} & \textbf{Avg.} \\
            \hline \hline
            ResNet-50-GN &54.5 & 52.8&53.1& 44.3& 21.2& 49.7& 39.2&54.8 & 54.0& 55.8& 75.4& 69.8& 59.6& 59.7& 66.3& 54.0\\
            {Tent}   & 58.8 & 58.5 & 58.7 & 38.2 & 26.8  & 54.9 & 42.6 & 51.6 & 38.8 & 61.9  & 75.3 & 70.0 & 62.3  & 63.6 & 66.3 & 55.2 \\
            {SAR}    & 60.3 & 59.6 & 59.5 & 46.6 & 33.0 & 57.5 & 47.8 & 57.8 & 52.8 & 65.1 & 76.7 & 71.4 & 67.3 & 66.0& 67.8 & 59.3 \\
            {\deyoabb}   & 64.4 & 64.5 & 63.7 & 55.2 & 46   & 63.1 & 55.9 & 62.3 & 55.8 & 69.8 & 77.0 & 73.5 & 71.5 & 70.7 & 70.2 & 64.2 \\
            \hline
            {\ens}       & 65.3 & 65.7 & 63.9 & 56.2 & 45.9 & 64.0 & 56.3 & 62.5 & 56.2 & 71.1 & 77.2 & 74.2 & 71.4 & 71.2 & 71.2 & 64.8 \\
            {\svgd}      & 66.6 & 65.6 & 65.0 & 56.8 & 46.5 & 64.5 & 57.9 & 63.8 & 57.3 & 72.6 & 77.8 & 74.5 & 71.6 & 72.6 & 71.4 & \underline{65.7} \\
            {\kl}        & 65.1 & 66.1 & 65.2 & 57.1 & 46.2 & 64.5 & 57.2 & 64.0 & 57.7 & 72.4 & 77.6 & 75.0 & 71.7 & 71.7 & 71.4 & {65.5} \\
            \rowcolor{LightCyan}
            \grad        & 67.2 & 67.8 & 65.7 & 58.0 & 46.9 & 64.3 & 58.5 & 65.6 & 57.5 & 73.0 & 77.7 & 76.2 & 72.1 & 72.5 &72.0 & \textbf{66.3} \\
            \hline
            \hline
            VitBase-LN & 51.6& 46.9& 50.5& 48.7& 37.2& 54.7& 41.6& 35.1& 33.5& 67.8& 69.3& 74.8& 65.8& 66.0& 63.7& 53.8\\
            {Tent}  & 67.1 & 66.2 & 66.3 & 66.3 & 60.9 & 69.1 & 61.4 &  65.2 & 60.4 & 75.2 & 78.1 & 78.8 & 74.9 & 75.8 & 72.4 & 69.2 \\
            {SAR}   & 68.5 & 67.8 & 68.0 & 67.8 & 63.1 & 70.7 & 63.5 & 66.9 & 62.8 & 75.8 & 77.7 & 78.4 & 74.7 & 75.7 & 72.7 & 70.3 \\
            {\deyoabb}  & 72.3 & 72.1 & 71.9 & 71.1 & 69.4 & 74.2 & 69.3 & 72.8 & 70.1 & 78.6 & 80.7 & 84.4 & 78.6 & 79.2 & 77.2 & 74.8 \\
            \hline
            {\ens}      & 73.1 & 72.8 & 72.3 &71.0  & 70.1 & 74.5 & 69.6 & 72.5 & 71.2 & 79.1 & 81   & 84.5 & 79.0 & 79.5 & 77.7 & 75.2 \\
            {\svgd}     & 74.5 & 74.0 & 72.4 & 71.3 & 71.1 & 75.3 & 70.1 & 73.7 & 71.8 & 79.7 & 82.0 & 84.8 & 80.1 & 79.8 & 79.1 & \underline{76.0} \\
            {\kl}       & 73.8 & 74.1 & 72.4 & 71.1 & 70.8 & 75.2 & 69.9 & 73.5 & 72.2 & 79.7 & 82.3 & 84.7 & 80.3 & 80.0 & 78.6 & {75.9} \\
            \rowcolor{LightCyan}
            \grad       & 74.8 & 74.8 & 73.1 & 71.2 & 71.0 & 75.1 & 71.3 & 73.4 & 73.1 & 79.7 & 82.2 & 85.0 & 81.2 & 79.8 & 78.8 & \textbf{76.3} \\
            \hline
        \end{tabular}
    }

    \smallskip
    
    \scalebox{0.6}{
        \begin{tabular}{l|ccc|cccc|cccc|cccc|c}
            \multicolumn{1}{c}{} & \multicolumn{3}{c}{Noise} & \multicolumn{4}{c}{Blur} & \multicolumn{4}{c}{Weather} & \multicolumn{4}{c}{Digital} & \\
            \textbf{Label Shifts} & \textbf{Gauss.} & \textbf{Shot} & \textbf{Impl.} & \textbf{Defoc.} & \textbf{Glass} & \textbf{Motion} & \textbf{Zoom} & \textbf{Snow} & \textbf{Frost} & \textbf{Fog} & \textbf{Brit.} & \textbf{Contr.} & \textbf{Elastic} & \textbf{Pixel} & \textbf{JPEG} & \textbf{Avg.} \\
            \hline \hline
            ResNet-50-GN & 54.5& 52.9& 53.1& 44.4& 21.2& 49.8&39.3 &54.9 & 54.1& 55.8& 75.3& 69.7& 59.6& 59.7& 66.4& 54.1\\
            {Tent} & 59.1 & 58.6 & 58.3 & 39.0 & 27.9 & 54.7 & 41.1 & 51.3 & 41.4 & 62.0 & 75.2 & 70.1 & 62.3 & 63.7 & 66.4 & 55.4 \\
            {SAR} & 60.8 & 60.5 & 60.2 & 47.9 & 36.7 & 58.2 & 49.7 & 57.9 & 53.6 & 65.0 & 76.4 & 71.0 & 67.0 & 65.8 & 67.6 & 59.9 \\
            {\deyoabb} & 64.9 & 63.9 & 63.3 & 54.0 & 44.9 & 62.2 & 55.1 & 61.2 & 57.9 & 69.2 & 76.9 & 73.2 & 71.2 & 70.2 & 69.8 & 63.9 \\
            \hline
            {\ens}     & 65.0 & 64.0 & 64.1 & 55.1 & 44.9 & 62.3 & 55.6 & 61.4 & 58.0 & 70.2 & 77.0 & 75.0 & 73.1 & 71.3 & 70.1 & 64.5 \\
            {\svgd}    & 65.6 & 64.4 & 64.3 & 55.9 & 46.1 & 62.9 & 56.0 & 61.7 & 59.1 & 73.4 & 78.1 & 76.5 & 73.7 & 71.9 & 71.2 & \underline{65.4} \\
            {\kl}      & 65.6 & 65.7 & 64.6 & 55.6 & 44.9 & 63.6 & 55.2 & 62.0 & 59.1 & 75.8 & 77.7 & 76.8 & 73.6 & 72.8 & 71.4 & 65.6 \\
            \rowcolor{LightCyan}
            \grad      & 65.6 & 65.6 & 65.1 & 55.9 & 45.9 & 65.1 & 56.7 & 63.1 & 59.1 & 77.1 & 78.6 & 76.9 & 73.8 & 72.6 & 72.8 & \textbf{66.3} \\
            \hline
            \hline
            ViTBase-LN & 51.5& 46.8& 50.4& 48.7& 37.1& 54.7& 41.6& 35.1& 33.3& 68.0& 69.3& 74.9& 65.9& 66.0& 63.6&53.8 \\
            {Tent} & 68.7 & 68.0 & 68.1 & 68.2 & 63.8 & 70.9 & 63.8 & 67.6 & 41.9 & 76.3 & 78.8 & 79.5 & 75.9 & 76.7 & 73.7 & 69.5 \\
            {SAR} & 68.8 & 68.2 & 68.4 & 68.3 & 64.7 & 71.0 & 64.2 & 68.1 & 66.0 & 76.4 & 79.0 & 79.6 & 76.2 & 77.1 & 74.1 & 71.3 \\
            {\deyoabb} & 71.7 & 71.6 & 71.4 & 70.6 & 68.9 & 73.9 & 69.2 & 72.4 & 69.7 & 77.9 & 80.2 & 79.5 & 78.2 & 78.7 & 76.7 & 74.0 \\
            \hline
            {\ens}     & 72.3 & 72.0 &71.5  & 72.1 & 71.7 & 74.0 & 70.2 & 74.0 & 72.0 & 78.8 & 81.1 & 80.6 & 80.1 & 79.9 & 77.0 & 75.2 \\
            {\svgd}    & 72.1 & 73.1 & 72.4 & 73.1 & 72.3 & 74.9 & 71.3 & 73.9 & 73.1 & 78.9 & 81.2 & 82.3 & 80.6 & 80.5 & 77.3 & \underline{75.8} \\
            {\kl}      & 73.2 & 73.1 & 72.3 & 72.9 & 71.8 & 74.7 & 71.3 & 74.1 & 72.5 & 77.9 & 81.2 & 81.5 & 80.7 & 80.5 & 77.2 & 75.7 \\
            \rowcolor{LightCyan}
            \grad      & 74.52 & 73.5 & 74.1& 73.8 & 72.6 & 74.7 & 71.3 & 74.7 & 72.9 & 79.0 & 81.2 & 82.1 & 80.6 & 81.2 & 77.3 & \textbf{76.2} \\
            \hline
        \end{tabular}
    }
\end{table*}

\begin{table*}[!ht]
\caption{Runtime comparison of TTA methods on ImageNet-C (Gaussian noise, severity 5) using ResNet-50-BN.}
\label{tab:runtime}
\centering
\scalebox{0.7}{
\begin{tabular}{lcccc} 
\toprule
Method & \#Forward & \#Backward & Other computation & GPU time (50 000 images) \\
\midrule
No adapt. & 50\,000 & --      & n/a                         & 79 s  \\
Tent      & 50\,000 & 50\,000 & n/a                         & 91 s  \\
SAR       & 68\,608 & 31\,099 & Additional model updates    & 139 s \\
DeYO      & 83\,843 & 25\,729 & Aug. filtering              & 134 s \\
\hline
\multirow{3}{*}{\grad} 
          & 83\,697 & 25\,469 &  &   \\[2pt]
          & 83\,621 & 25\,339 & Aug. filtering + \cref{eq:grad_div} & 351 s \\[2pt]
          & 83\,573 & 25\,211 &  &   \\
\bottomrule
\end{tabular}}
\end{table*}

\cref{tab:imagenet-sev3} presents a comparison between the baseline methods and our diversified variants on ImageNet-C at severity level 3. The reported accuracy values are averaged over five different random seeds. It is evident that all diversified methods outperform Deyo (the previous state-of-the-art) by 2–3\%, showing the effectiveness of our ensembling-like diversification strategy in mitigating the effects of underspecification.

\section{Runtime} \label{append:runtime}

\cref{tab:runtime} provides a runtime comparison of TTA methods. We evaluate the computational cost of test-time adaptation techniques by benchmarking their performance on the ImageNet-C dataset. Specifically, we use the ResNet-50-BN architecture and focus on the Gaussian noise corruption at severity level 5, comprising 50,000 images. All timing experiments are conducted on a single NVIDIA ADA A6000 GPU. Among the compared methods, SAR employs a two-step optimization procedure that involves filtering out high-entropy samples before performing adaptation. DeYO, on the other hand, combines augmentation filtering with entropy-based sample exclusion to reduce computational cost while improving stability. Our framework builds upon DeYO by introducing diversification strategies at multiple levels. In particular, for the case of three particles, the total runtime remains less than three times that of DeYO. This is because not all particles undergo adaptation for every batch. Due to the induced diversity, the particles are positioned in distinct regions of the parameter space and thus respond differently to incoming data.

\section{Limitation}
\label{append:limitation}

The limitation of our proposed wrapper lies in its increased inference overhead, which scales roughly linearly with the number of particles. For example, utilizing three particles can increase runtime by nearly a factor of three compared to a single model. This added cost arises from the diversification mechanism, which drives the particles to explore distinct regions of the parameter space and generate complementary predictions. While this results in greater computational overhead, the wrapper consistently delivers improved accuracy, making it especially suitable for applications where performance takes precedence over inference speed.

%
\end{document}